%% file: main.tex
\documentclass{article} %
\usepackage{arxiv}

\usepackage{microtype}
\usepackage{hyperref}
\usepackage{url}
\usepackage{amsmath}
\usepackage{amsfonts}
\usepackage{booktabs}
\usepackage{cleveref}
\usepackage{graphicx}
\graphicspath{{./}{../}}
\usepackage{amssymb}
\usepackage{algorithm}
\usepackage[noend]{algorithmic}
\usepackage{wrapfig}
\usepackage[normalem]{ulem}
\usepackage{pifont}
\usepackage{subcaption}
\input{math_commands}

\usepackage{tcolorbox, listings, bm}
\tcbuselibrary{skins, breakable, listings, raster}

\usepackage{lineno}
\usepackage{xspace}
\definecolor{scggreen}{HTML}{DCEEF0}
\definecolor{scgborder}{HTML}{458588}
\definecolor{norevred}{HTML}{E4F0E4}
\definecolor{norevborder}{HTML}{689D6A}
\definecolor{seedgray}{HTML}{F1F3F5}
\definecolor{seedborder}{HTML}{6C757D}
\definecolor{orhighlight}{HTML}{C62828}
\definecolor{codeblue}{HTML}{1565C0}
\definecolor{codegray}{HTML}{455A64}
\definecolor{conclusionhl}{HTML}{1565C0}
\definecolor{ours}{HTML}{51969A}
\definecolor{RL}{HTML}{D65E0D}
\definecolor{solved}{HTML}{98971A}
\definecolor{notsolved}{HTML}{cc241d}

\definecolor{myblue}{rgb}{0.21,0.49,0.74}
\definecolor{darkblue}{rgb}{0, 0, 0.5}
\hypersetup{colorlinks=true, citecolor=myblue, linkcolor=myblue, urlcolor=myblue}

\usepackage[dvipsnames]{xcolor}

\definecolor{PromptLightBlue}{HTML}{eefcfd}
\definecolor{PromptBlue}{HTML}{44BDC6}

\newtcolorbox{prompt}[2][PromptBlue]{%
    enhanced,
    colback=PromptLightBlue,        
    colframe=#1,                    
    colbacktitle=#1,                
    coltitle=white,                 
    fonttitle=\bfseries,
    title=#2,
    arc=4pt,
    boxrule=0.8pt,
    left=8pt, right=8pt,
    top=6pt, bottom=6pt,
    width=0.95\linewidth,
    center,
    attach boxed title to top left={xshift=10pt, yshift=-\tcboxedtitleheight/2},
    boxed title style={
        colframe=#1,
        arc=3pt,
        boxrule=0pt,
    },
}

\definecolor{FrontBoxGray}{HTML}{F1F1F1}

\newtcolorbox{frontmatterbox}{%
    enhanced,
    colback=FrontBoxGray,
    colframe=FrontBoxGray,
    boxrule=0pt,
    arc=18pt,
    outer arc=18pt,
    left=28pt,
    right=28pt,
    top=30pt,
    bottom=28pt,
    width=\textwidth,
    center,
    before skip=1.5em,
    after skip=2em,
}

\title{Harnesses for Inference-Time Alignment over \\ Execution Trajectories}

\author{Boyuan Wang\textsuperscript{$\diamond$}, Bochao Li\textsuperscript{$\diamond$}, Minghan Wang\textsuperscript{$\diamond$}, Yuxin Tao\textsuperscript{$\dagger$}, Fang Kong\textsuperscript{$\dagger$}}

\usepackage{xcolor}
\definecolor{infra-blue}{HTML}{0076BA}
\definecolor{infra-fault}{HTML}{F27300}

\begin{document}

\maketitle

\begin{center}
    \vspace{-2mm}
   Southern University of Science and Technology  
\end{center}

\input{text/0_abs}


\footnotetext{\textsuperscript{$^{\diamond}$}Equal contribution.}
\footnotetext{
\textsuperscript{$^\dagger$}Correspondence to 
\href{taoyx@sustech.edu.cn}{taoyx@sustech.edu.cn} and
\href{kongf@sustech.edu.cn}{kongf@sustech.edu.cn}.
}

\input{text/1_intro}
\input{text/2_related_work}
\input{text/3_preliminary}

\input{text/4_main_framework}
\input{text/5_experiment}
\input{text/6_conclusion}

\bibliographystyle{plainnat}
\bibliography{main}

\appendix
\clearpage

\input{text/999_apendix}

\clearpage
\newpage

\end{document}

%% file: math_commands.tex
\usepackage{amsfonts}
\usepackage{amsmath}
\usepackage{amssymb}

\usepackage{amsthm}

\usepackage{thmtools}
\usepackage{thm-restate}

\usepackage{bm}
\usepackage{float}
\usepackage{graphicx}
\usepackage{caption}
\usepackage{hyperref}
\usepackage{cleveref}
\usepackage{nccmath}
\usepackage{color}
\usepackage{algorithm}
\usepackage{algorithmic}
\usepackage{enumitem}
\usepackage{fancyhdr}
\usepackage{booktabs}

\usepackage{mathtools}
\usepackage{latexsym}
\usepackage{dsfont}
\usepackage{mathrsfs}
\usepackage{xspace}

\theoremstyle{plain}
\newtheorem{theorem}{Theorem}%

\newtheorem{lemma}[theorem]{Lemma}
\newtheorem{corollary}[theorem]{Corollary}

\newtheorem{assumption}[theorem]{Assumption}

\theoremstyle{definition}
\newtheorem{remark}[theorem]{Remark}

\def\eqref#1{equation~\ref{#1}}

\def\1{\bm{1}}

\DeclareMathAlphabet{\mathsfit}{\encodingdefault}{\sfdefault}{m}{sl}
\SetMathAlphabet{\mathsfit}{bold}{\encodingdefault}{\sfdefault}{bx}{n}

\usepackage{bm}

%% file: text/0_abs.tex
\begin{abstract}
Harness engineering has emerged as an important inference-time technique for large language model (LLM) agents, aiming to improve long-term performance through task decomposition and guided execution. However, \textbf{more elaborate harnesses are not uniformly better}: increasing decomposition or guidance can sometimes improve execution, but can also reduce final task success. We study harness design through the lens of inference-time trajectory alignment. This perspective separates harness into two mechanisms: task decomposition, which structures a task into sub-goals, and guided execution, which reshapes local action distributions during execution. This decomposition allows us to quantify how workflow granularity, retry budgets, and guidance-induced action reweighting shape the performance limits of harness design. It further reveals concrete failure modes, including over-decomposition, over-pruning, and hallucinated execution. We validate these predictions through controlled synthetic experiments and real terminal agent benchmarks. Inspired by the theory, we further show that \textbf{effective harnesses can be partial}: specifying only the initial steps and leaving the remaining execution to agent can achieve higher pass rate than fully structured workflows.
\end{abstract}

%% file: text/1_intro.tex
\section{Introduction}

Large language model agents have demonstrated impressive capabilities in solving long-horizon interactive tasks involving complex software engineering \citep{claudecode2026, codexcli2026}, scientific discovery \citep{qu2026crispr, jin2025stella}, and autonomous tool use \citep{zeng2026glm, team2025kimi}. 
A key technique behind is \emph{harnessing}: a scaffolding strategy that injects human priors into the agent's execution process \citep{openai2026harnessengineering, Langchain2026, Anthropic_harness}. 
By decomposing long-horizon tasks into structured sub-goals and providing guidance for intermediate decisions, a harness enables agents to complete complex tasks more reliably in an autonomous manner \citep{autoresearch, agentsmd, Harbor_Framework, BrowserHarness}.

This success suggests a tempting intuition: more elaborate harnesses should lead to better agents.
From this view, harness design becomes a problem of adding structure: finer-grained decomposition, more detailed instructions, and tighter execution constraints \citep{erdogan2025planandact, Wang2025DyFlowDW, dang2025multi}.
Yet this intuition conflicts with a central lesson from the history of AI \citep{bitterlesson}: human-designed structure often helps in the short run, but can limit a system's ability to search, adapt, and scale \citep{yan2025reformreducinghuman, guo2025deepseek, Silver2017MasteringTG}. 
This raises a basic question: \textbf{What should a harness specify, and what should leave for agent to resolve on its own?}

To answer this question, we theoretically model how harnesses shape the inference-time trajectories of LLM agents.
We separate a harness into two components: \emph{workflow}, which specifies what the agent should achieve at each stage, and \emph{guidance}, which biases how the agent acts within that stage.
This separation yields two findings through stage-level gaps.
First, \textbf{finer-grained decomposition is not always better}. The optimal granularity must align the required sub-goal scale with the agent's controllable progress scale under the given tolerance and retry budget.
Second, \textbf{guidance helps only when aligned}. It improves performance by shifting probability mass toward recoverable actions. Moreover, when guidance favors actions that follow the instruction rather than the task evidence, it can instead lead to hallucinated responses. These findings recast harness design as an alignment problem, where imposed structure must match the agent's capability and the available task evidence.

Beyond these observations, we find that although a harness is usually designed for the whole task, keeping only its early stages can be more effective.
This points to a counterintuitive principle for harness design: effective harnesses need not specify the full execution path. We formalize this idea as \textbf{Partial Harnessing}, a new design strategy that specifies only the initial stages and leaves the remaining execution to the agent. Our experiments show that partial harnessing can outperform fully specified workflows, suggesting that harness design should decide not only what structure to add, but also when to stop adding it.

We summarize our contributions as follows:
\begin{itemize}[leftmargin=*, itemsep=2pt, topsep=2pt]
    \item We formulate harness design as an inference-time alignment problem by decomposing a harness into workflow and guidance components, yielding two stage-level principles: sub-goal scale should match agent capability, and guidance should match task evidence.
    \item We introduce {partial harnessing} as a design strategy that stops scaffolding once its reliability cost outweighs tail-risk reduction, formalized through a marginal stopping rule.
    \item We empirically validate these predictions on synthetic cumulative-progress tasks and Terminal-Bench v2, where alignment patterns appear and partial harnesses outperform full workflows.
\end{itemize}

\begin{figure}[t]
    \centering
    \includegraphics[width=1\linewidth]{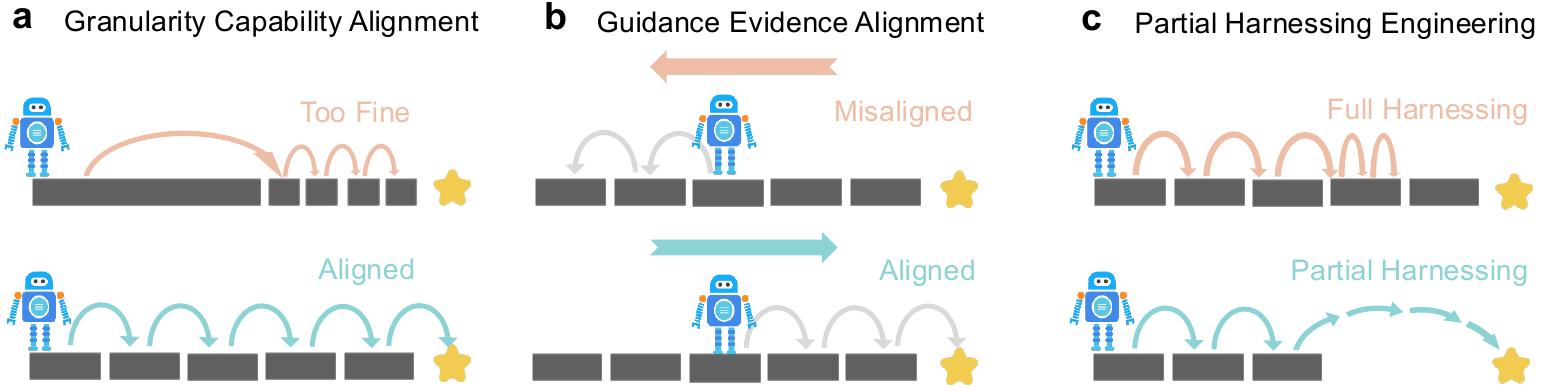}
    \caption{Alignment Principles for Harnessed Agent Execution}
    \label{fig:placeholder}
\end{figure}

%% file: text/2_related_work.tex
\section{Related Work}
\label{sec:related_work}

\textbf{Long-horizon agent execution.}
LLM agents have grown from single-step responders into systems capable of executing long-horizon tasks in interactive environments. The starting point is the reasoning-action loop of ReAct~\citep{yao2022react}, which later work extends with feedback and memory~\citep{shinn2023reflexion, packer2023memgpt} and with explicit planning and search~\citep{yao2023tree, zhou2023language}. As these capabilities matured, evaluation moved into realistic long-horizon environments spanning web interaction~\citep{zhou2023webarena}, software engineering~\citep{jimenez2023swe}, and computer use~\citep{xie2024osworld}, prompting recent systems to treat the surrounding scaffold itself as a design object and optimize the workflow that orchestrates execution~\citep{hong2023metagpt, zhang2024aflow}. A common thread across this progression is that every step adds more human-designed structure around the model. Our work asks what such structure should actually specify, and what should be left to the agent.

\textbf{Harness optimization.}
A related line of work treats the scaffolds around LLM execution as optimization targets. Early methods optimize prompts directly, by generating, scoring, or refining natural-language instructions \citep{zhou2022large, pryzant2023automatic, yang2024opro, fernando2023promptbreeder}. This view is extended from single prompts to multi-stage LM programs, where systems such as DSPy, MIPRO, Self-Refine, and TextGrad optimize instructions, demonstrations, or intermediate artifacts across multiple model calls \citep{khattab2023dspy, opsahl2024optimizing, madaan2023self, yuksekgonul2024textgrad}. More recent work broadens the search space to agent modules, executable workflows, and full harness implementations \citep{shang2024agentsquare, hu2024automated, agrawal2025gepa, novikov2025alphaevolve, lee2026meta}. These works primarily ask how to discover better scaffolds. In contrast, we ask when additional scaffold should help at all: which decompositions and guidance rules align with the agent's execution, and which ones over-specify the trajectory.

\textbf{Human priors and agent autonomy.}
A broader line of work studies how control should be shared between humans and autonomous systems. Mixed-initiative interaction and adjustable autonomy show that authority need not belong entirely to either side, but can be allocated according to uncertainty, context, and task demands \citep{Horvitz1999PrinciplesOM, scerri2002towards}. Recent LLM agent systems revisit this issue as humans provide goals, constraints, feedback, or oversight while agents execute increasingly long-horizon tasks \citep{feng2024large, zou2025call, wang2025interaction}. Harness design is a concrete instance of this trade-off: it specifies part of the trajectory through decomposition and guidance, while leaving the rest for the agent to resolve \citep{pan2026natural, bui2026building} . Our work studies when such human priors improve execution and when they over-constrain the agent's execution.

%% file: text/3_preliminary.tex
\section{Preliminary}
\label{sec:preliminary}

\textbf{Intuition.}
We think of solving a task as proceeding on two interacting timescales,
one driven by the harness and one driven by the agent. The \emph{outer}
timescale belongs to the harness, which lays out a sequence of sub-goals
$g_1 \to g_2 \to \cdots \to g_T$ toward the final answer and advances along
this plan one sub-goal at a time. The \emph{inner} timescale belongs to the
agent: given the current sub-goal $g_t$, it repeatedly takes an action and
observes the resulting state until $g_t$ is met, at which point control
returns to the harness and $g_{t+1}$ is revealed. Crucially, the harness is
not silent during this inner loop: even as the agent chooses each step, the
harness simultaneously \emph{shapes} the loop, nudging the trajectory toward
behaviors it considers promising for $g_t$. The two timescales thus carry a
clean division of labor: the harness decides \emph{what to work on next} and
\emph{which trajectories to favor while working on it}, while the agent
decides \emph{how to take each step}.

\textbf{How a Harness Shapes Execution.}
We now make two timescales precise. Consider tasks $x \sim \mathcal{D}$, each with a unique correct final answer $y^\star(x)$, and define a \emph{harness} as an inference-time scaffold parameterized by $h=(\kappa,\lambda,\psi)$, where $\kappa$ controls the decomposition granularity, $\lambda$ controls the guidance strength, and $\psi$ specifies the local guidance rule. These three parameters split cleanly along the two timescales: $\kappa$ governs the outer one, and $\lambda, \psi$ govern the inner one.

For the outer timescale, $\kappa$ applied to a task $x$ induces the ordered sub-goal sequence $\Delta_h(x) = (g_1,\ldots,g_{T_h(x)})$, the coarse plan from the intuition above. This sequence fixes the stage-level structure but leaves each inner trajectory to the agent, so a complete harness-conditioned execution takes the form
$$
\tau_h(x) = (g_1,\tau_1,\ldots,g_T,\tau_T),
\qquad T = T_h(x),
$$
where each $\tau_t = (s_{t,0},a_{t,0},\ldots,a_{t,n_t-1},s_{t,n_t})$ is the inner trajectory generated while pursuing $g_t$, with states $s_{t,j}\in\mathcal{S}$ and actions $a_{t,j}\in\mathcal{A}$. How each $\tau_t$ is actually produced is the job of the inner timescale.

For the inner timescale, let $K_{t-1}$ collect everything observed before stage $t$, i.e.\ the task together with all preceding sub-goals and trajectories. Without guidance, the agent unrolls $\tau_t$ from $K_{t-1}$ and $g_t$ under its \emph{base trajectory distribution} $\mathbb{Q}_{t,0}(\tau_t \mid K_{t-1}, g_t)$, defined autoregressively by $a_{t,j} \sim q_h(\cdot \mid H_{t,j})$ with $H_{t,j} = (K_{t-1}, g_t, s_{t,0}, a_{t,0}, \ldots, a_{t,j-1}, s_{t,j})$. The base distribution captures the agent acting on its own; the harness's role is to reshape it.

This reshaping is realized through $\psi$ and $\lambda_t$, which together determine a non-negative weight $W_{t,\lambda_t}(K_{t-1}, g_t, \tau_t)$ measuring how well a candidate trajectory aligns with the behavior $\psi$ prescribes for $g_t$. Reweighting the base distribution by this score yields the \emph{guided trajectory distribution},
$$
\mathbb{Q}_{t,\lambda_t}(\tau_t \mid K_{t-1}, g_t)
\;\propto\;
\mathbb{Q}_{t,0}(\tau_t \mid K_{t-1}, g_t)\,
W_{t,\lambda_t}(K_{t-1}, g_t, \tau_t),
$$
which governs the agent's actual behavior at stage $t$. The strength $\lambda_t$ controls the magnitude of this reweighting: at $\lambda_t = 0$ the weight is uniform and the guided distribution collapses to the base distribution, and as $\lambda_t$ grows, $\mathbb{Q}_{t,\lambda_t}$ concentrates on trajectories preferred under $\psi$. Once $\tau_t$ terminates, control returns to the harness, $g_{t+1}$ is revealed, and this two-level process repeats until stage $T$ is complete.

\textbf{From Final Success to Stagewise Recoverability.}
We now make this decomposition precise. Let $y(\tau_h)$ denote the final
answer produced by execution $\tau_h$, define the final success event as
$\mathrm{Succ}_x(\tau_h) := \{y(\tau_h) = y^\star(x)\}$, and write the
primitive harness-design objective as
$$
\max_h \; \mathbb{E}_{x \sim \mathcal{D}}
\left[
\mathbb{P}_h(\mathrm{Succ}_x(\tau_h) \mid x)
\right].
$$
To connect this terminal objective with process-level behavior, we introduce
the completed prefix after stage $t$,
$K_t := (x,g_1,\tau_1,\ldots,g_t,\tau_t)$ with $K_0 := x$, and let $B_t$
denote the event that $K_t$ is \emph{recoverable}, i.e.\ that some
continuation under the remaining plan $(g_{t+1},\ldots,g_T)$ still reaches
$y^\star(x)$. The intuition above says that final success is equivalent to
recoverability holding throughout, which under goal consistency we write as
$\mathrm{Succ}_x(\tau_h) \equiv \bigcap_{t=1}^{T_h(x)} B_t$.

This equivalence turns the terminal objective into a stagewise product. By
the chain rule,
$$
\mathbb{P}_h(\mathrm{Succ}_x(\tau_h) \mid x)
=
\prod_{t=1}^{T_h(x)}
\bar p_t(h;x),
\qquad
\bar p_t(h;x)
:=
\mathbb{P}_h(B_t \mid B_{<t}, x),
$$
where $B_{<t}:=\bigcap_{s<t}B_s$ and $B_{<1}:=\Omega$, so $\bar p_t(h;x)$ is
the conditional probability of remaining recoverable at stage $t$ given that
the run was recoverable through every earlier stage. Taking negative
logarithms turns the product into a sum,
$-\log \mathbb{P}_h(\mathrm{Succ}_x(\tau_h) \mid x) = \sum_{t=1}^{T_h(x)} -\log \bar p_t(h;x)$,
which exhibits the process loss as the stagewise decomposition of the
primitive final-success objective rather than an auxiliary term added on
top of it.

%% file: text/4_main_framework.tex
\section{Alignment Principles for Harness Design}
\label{sec:structure_theorems}

Building on the recoverability framework of Section~\ref{sec:preliminary}, we now present three alignment principles, each addressing a distinct lever the harness has over execution: granularity--capability alignment in Section~\ref{sec:granularity_capability_alignment}, guidance--evidence alignment in Section~\ref{sec:guidance_action_space_filtering}, and partial harnessing in Section~\ref{sec:partial-harnessing}.

\subsection{Granularity - Capability Alignment}
\label{sec:granularity_capability_alignment}

A harness's outer timescale specifies sub-goals, but it does not specify them in the abstract: each sub-goal asks the agent to make a definite amount of progress within a finite execution budget. Whether a workflow helps therefore depends on a single relationship---between the progress each stage requests and the progress the agent can reliably realize at that stage. We characterize this relationship precise and show that it controls the final success probability.

Consider stage $t$ in isolation. The harness requests latent progress $\ell_t$ and grants the agent at most $M_t$ low-level steps in which to deliver it. After $m\le M_t$ such steps, the cumulative progress the agent can reliably control lies in a window $I_{t,m}=[\mu^-_{t,m},\mu^+_{t,m}]$ with stochastic variation $\sigma_{t,m}$, and a tolerance $\epsilon_t$ describes how far the realized progress may drift from $\ell_t$ before the resulting prefix ceases to be recoverable. The relevant quantity is the smallest standardized gap between the requested progress and any scale the agent can reach within budget,
$$
\rho_t^{(M_t)}
:=
\min_{1\le m\le M_t}
\frac{\big(d(\ell_t,I_{t,m})-\epsilon_t\big)_+^2}{2\sigma_{t,m}^2},
$$
which vanishes when some reachable cumulative scale lies within tolerance of $\ell_t$ and grows quadratically in the gap when no such scale exists. The mismatch $\rho_t^{(M_t)}$ is therefore a stage-local property: it depends only on what the harness asks of stage $t$ and what the agent can do within $M_t$ steps.

Aggregating these stage-local mismatches across the workflow yields the main result.
\begin{theorem}[Granularity-capability mismatch bound, informal]
\label{thm:granularity_capability_alignment_informal}
Consider a harness that decomposes a task $x$ into $T=T_h(x)$ sub-goals. At
stage $t$, suppose the harness requires latent progress $\ell_t$, and the agent
may take at most $M_t$ low-level execution steps. Under the
recoverability-tube, concentration, and boundary-contraction conditions in
Appendix~\ref{sec:granularity_alignment}, the final success probability
satisfies
$$
    \mathbb{P}_h(\mathrm{Succ}_x(\tau_h)\mid x)
    \le
    \exp\left(
    -
    \sum_{t=1}^{T_h(x)}
    \left[
        \eta_t+
        \bigl(\rho_t^{(M_t)}-\log M_t\bigr)_+
    \right]
    \right).
$$
\end{theorem}
The bound exposes two distinct contributions to per-stage loss. The execution cost $\eta_t$ reflects the intrinsic difficulty of any one stage and is paid regardless of how the workflow is structured. 
By contrast, the granularity penalty $\bigl(\rho_t^{(M_t)}-\log M_t\bigr)_+$ appears only when the requested scale exceeds what the retry budget can absorb, and it precisely captures the cost of misaligned decomposition.
A larger $M_t$ enlarges the reachable set and softens this penalty through the $-\log M_t$ term, but cannot eliminate a structural gap between $\ell_t$ and the agent's execution dynamics.

This decomposition explains why finer workflows are not uniformly better. Specialize to a uniform $T$-stage workflow on a task with total latent progress $L_x$, so each stage requests $L_x/T$. If a single low-level step reliably advances progress by an amount in $[\mu^-,\mu^+]$ and each stage allows at most $M$ steps, the cumulative scales reachable within a stage form the union
$$
\bigcup_{m=1}^{M} [m\mu^-,m\mu^+],
$$
and useful decompositions place $L_x/T$ near this set. Small $T$ pushes $L_x/T$ above the union and leaves stages unreachable within $M$ steps; large $T$ pushes $L_x/T$ below it and imposes milestones the agent cannot stop at without coordination loss. The reliable regime lies between these extremes, where each stage requires nontrivial progress while remaining reachable under the agent's realizable progress scale. Workflow granularity is thus a design parameter to be chosen against the agent's execution dynamics, not against the task's logical structure alone.

\subsection{Guidance - Evidence Alignment}
\label{sec:guidance_action_space_filtering}

The inner timescale of a harness is shaped by guidance: at stage $t$, the weight $W_{t,\lambda_t}$ reweights the agent's base trajectory distribution toward behaviors $\psi$ prescribes for $g_t$. This reweighting is helpful when it concentrates probability on trajectories that keep the prefix recoverable, and harmful when it concentrates probability on trajectories that look locally preferred but foreclose successful continuations. We make this dichotomy precise through a single stage-local quantity.

Fix a recoverable prefix $K_{t-1}$, and let $R_t^{\mathrm{stg}}(K_{t-1})$ denote the set of stage trajectories that keep $K_t$ recoverable. Under the base distribution, the average retention weight is $\bar W^{\mathrm{rec}}_{t,\lambda_t}(K_{t-1})$ on $R_t^{\mathrm{stg}}$ and $\bar W^{\mathrm{bad}}_{t,\lambda_t}(K_{t-1})$ on its complement. The retention gap
$$
\Gamma_{t,\lambda_t}(K_{t-1})
:=
\log \bar W^{\mathrm{rec}}_{t,\lambda_t}(K_{t-1})
-
\log \bar W^{\mathrm{bad}}_{t,\lambda_t}(K_{t-1})
$$
measures how much more weight guidance places on recoverable trajectories than on non-recoverable ones, in log space. A positive gap means guidance preferentially preserves trajectories that keep a successful continuation available, whereas a negative gap means it preserves trajectories that may look locally acceptable but render the prefix non-recoverable.

This single sign-valued quantity determines whether guidance helps at the prefix.
\begin{theorem}[Guidance alignment via retention gaps, informal]
\label{thm:guidance_retention_gap}
Fix a task $x$, a stage $t$, and a recoverable prefix $K_{t-1}$. Under the regularity conditions in Appendix~\ref{app:guidance_action_space_filtering}, the stage recoverability probability after guidance satisfies
$$
Q_{t,\lambda_t}
\!\left(
R_t^{\mathrm{stg}}(K_{t-1})
\mid K_{t-1},g_t
\right)
=
\sigma\!\left(
\omega^0_t(K_{t-1})
+
\Gamma_{t,\lambda_t}(K_{t-1})
\right),
$$
where $\omega^0_t(K_{t-1})$ is the unguided recoverability log-odds and $\sigma(u)=1/(1+e^{-u})$. Consequently, guidance improves stage recoverability relative to the unguided law at the same prefix if and only if $\Gamma_{t,\lambda_t}(K_{t-1})>0$, and harms it if and only if $\Gamma_{t,\lambda_t}(K_{t-1})<0$.
\end{theorem}
The theorem reduces a multi-dimensional design choice---which trajectories to favor, and how strongly---to a one-dimensional diagnostic at each prefix. The unguided log-odds $\omega^0_t(K_{t-1})$ is a property of the agent and the prefix, fixed once both are specified. Guidance enters the recoverability probability only through $\Gamma_{t,\lambda_t}$, additively in log-odds space. Stronger guidance, formalized as a larger $\lambda_t$, scales the magnitude of $\Gamma_{t,\lambda_t}$ but not its sign, so increasing $\lambda_t$ amplifies whichever effect $\psi$ already produces at $K_{t-1}$: helpful guidance becomes more helpful, and misaligned guidance becomes more harmful.

This sign-amplification structure explains why the same guidance can either reduce or amplify hallucination. In evidence-limited settings, a recoverable trajectory is one that stays within what the current observations justify. A guidance rule $\psi$ that rewards evidence checking or uncertainty awareness places its weight on grounded trajectories, producing a positive retention gap; raising $\lambda_t$ then suppresses ungrounded continuations and reduces hallucination. A guidance rule $\psi$ that rewards detail, confidence, or instruction compliance without conditioning on evidence places weight on trajectories that satisfy $\psi$ regardless of whether they remain grounded, producing a negative retention gap; raising $\lambda_t$ then drives the agent further from the supported set, amplifying hallucination. Guidance strength is therefore not a generic reliability lever: its effect at a prefix is determined by the alignment between $\psi$ and the evidence available at that prefix.

\subsection{Partial Harnessing as a Marginal Reliability Trade-off}
\label{sec:partial-harnessing}
A harness's length, the number of stages it specifies before releasing the agent, is itself a design lever. Each additional scaffolded stage shortens the residual task left to the agent, but imposes another recoverability constraint that execution must satisfy. Adding a stage is helpful when its tail-risk reduction exceeds its reliability cost, and harmful when the cost dominates. We make this dichotomy precise through a single marginal comparison, and use it to characterize when \emph{partial harnessing}, which specifies only an initial prefix of the trajectory and leaves the rest to the agent, outperforms full coverage.

Fix a task with total latent progress demand \(L_x\) and a scaffold step size
\(s>0\). For \(m\in\mathcal J:=\{0,1,\ldots,\lfloor L_x/s\rfloor\}\), let
\(h_m\) denote the partial harness that specifies the first \(m\) scaffolded
stages and leaves the residual length \(L_x-ms\) to the autonomous agent. Two
scalar quantities determine the trade-off: the \emph{scaffold cost}
\[
c_s\;:=\;-\log q_{\mathrm{scaf}}(s;M),
\]
the negative log-reliability of executing one scaffolded stage, and the
\emph{tail risk}
\[
\kappa_{\mathrm{tail}}(d;M)\;:=\;-\log q_{\mathrm{tail}}(d;M),
\]
the negative log-reliability of finishing a residual task of length \(d\)
autonomously. Under the homogeneous slice factorization in
Appendix~\ref{app:partial-harnessing}, the negative log-success of \(h_m\)
separates additively as
\[
F(m)\;:=\;-\log\Pr\nolimits_{h_m}\!\bigl(\mathrm{Succ}_x(\tau_h)\mid x\bigr)
\;=\;m\,c_s\;+\;\kappa_{\mathrm{tail}}(L_x-ms;M),
\]
where the first term grows linearly with coverage while the second shrinks as
the autonomous tail becomes shorter.

The trade-off between these two terms collapses into a one-dimensional
diagnostic at the margin. Let
\[
\Delta(m;M)\;:=\;\kappa_{\mathrm{tail}}(L_x-ms;M)\;-\;\kappa_{\mathrm{tail}}(L_x-(m{+}1)s;M)
\]
denote the reduction in tail risk obtained by adding the \((m{+}1)\)-st
scaffolded stage. Adding that stage strictly improves reliability if and only
if this reduction exceeds its own scaffold cost,
\(
\Delta(m;M)\;>\;c_s.
\)
Partial harnessing thus follows a simple marginal principle: keep extending
the harness only while the next scaffolded stage saves more residual risk
than it introduces.

\begin{theorem}[Coverage--autonomy alignment, informal]
\label{thm:partial-harnessing}
Under the homogeneous slice factorization and the diminishing-returns
condition on \(\kappa_{\mathrm{tail}}\) stated in
Appendix~\ref{app:partial-harnessing}, \(F(m)\) is discrete-convex in \(m\),
and the success probability \(\exp(-F(m))\) is log-concave and unimodal over
\(m\in\mathcal J\). The smallest reliability-maximizing coverage is
\[
m_{\mathrm{peak}}
\;=\;
\min\{m\in\mathcal J:\;m+1\in\mathcal J,\;\Delta(m;M)\le c_s\},
\]
with \(m_{\mathrm{peak}}=\max\mathcal J\) if the set is empty. For a target
reliability \(\alpha\in(0,1)\), the minimum-structure \(\alpha\)-reliable
coverage is \(m_\alpha=\min\{m\in\mathcal J:F(m)\le-\log\alpha\}\) whenever
this set is nonempty.
\end{theorem}

The theorem connects harness length to the preceding alignment principles via a single marginal comparison. Far from a free parameter, the scaffold cost \(c_s\) is raised by granularity mismatch (Theorem~\ref{thm:granularity_capability_alignment_informal}) and misaligned guidance (Theorem~\ref{thm:guidance_retention_gap}), and lowered by well-aligned scaffolding. The tail risk \(\kappa_{\mathrm{tail}}\), by contrast, is a property of the agent. More capable agents attain smaller \(\kappa_{\mathrm{tail}}\) on any residual, so \(\Delta(m;M)\) crosses \(c_s\) at smaller \(m\) and the optimal coverage shifts earlier.

Harness length is therefore not a generic reliability lever, with its effect governed by the alignment between local scaffold quality and the agent's autonomous capability on the residual. The right harness is the smallest one that brings the residual within autonomous reach, and no smaller. The slice rule is a clean stopping principle rather than a universal optimization law. When added stages alter earlier behavior or tail success depends strongly on the realized prefix, the success curve may become multi-modal, and candidate coverages must be compared directly.

%% file: text/5_experiment.tex
\section{Experiments}

To validate our theoretical predictions on harness design, we conduct \textbf{synthetic experiments} for controlled mechanism analysis in Section~\ref{sec:synthetic} and \textbf{real-data experiments} for realistic long-horizon validation in Section~\ref{sec:real}, examining three aspects: (i) \emph{decomposition granularity}, (ii) \emph{guidance alignment}, and (iii) \emph{partial harness specification}.

\subsection{Synthetic Experiments}
\label{sec:synthetic}

\begin{figure}[t]
    \centering
    \begin{subfigure}{0.32\textwidth}
        \centering
        \includegraphics[width=\linewidth]{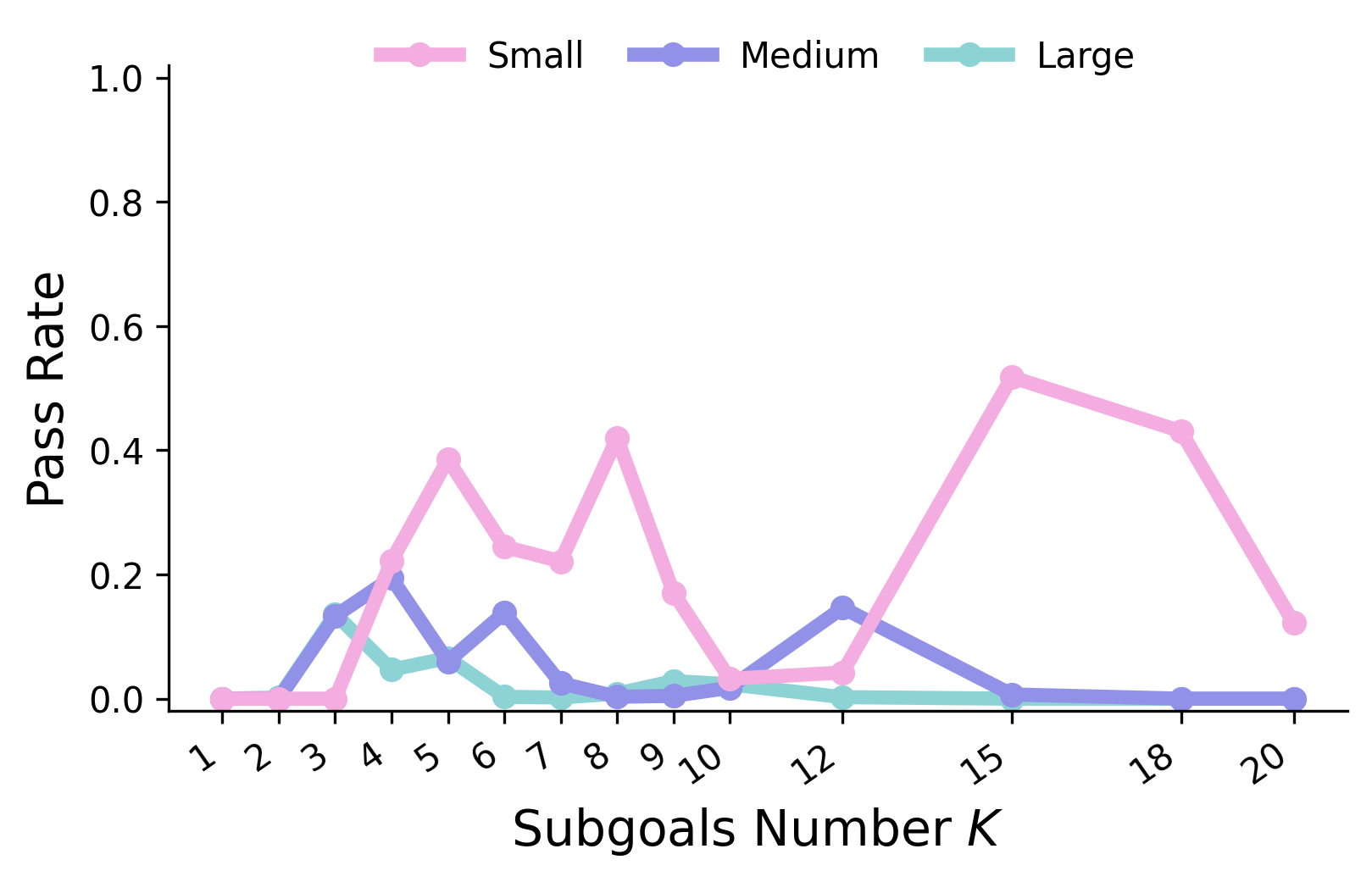}
        \caption{Pass rate under different granularities}
        \label{fig:exp1_a}
    \end{subfigure}
    \hfill
    \begin{subfigure}{0.32\textwidth}
        \centering
        \includegraphics[width=\linewidth]{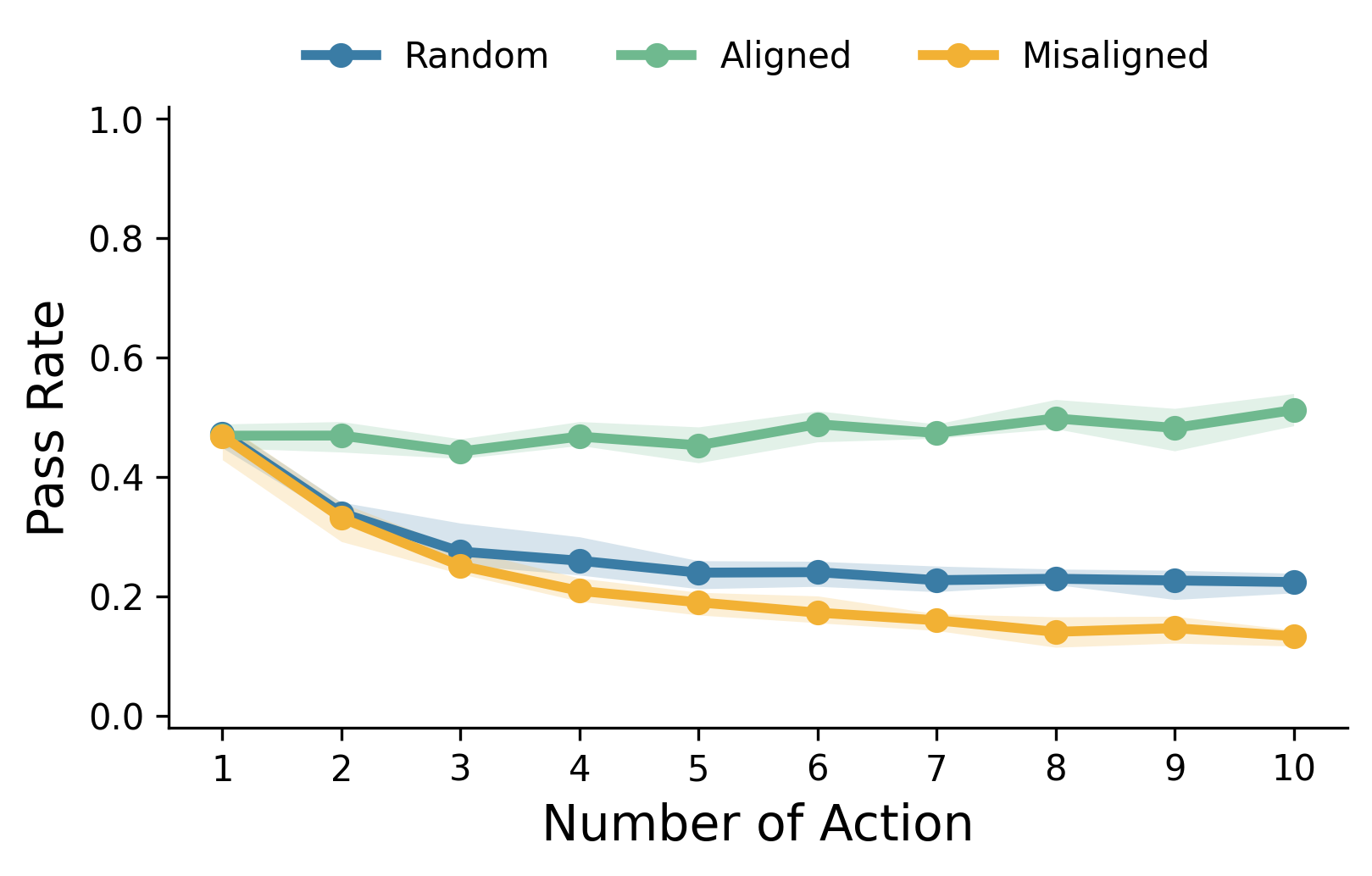}
        \caption{Pass rate under different guidance levels}
        \label{fig:exp2_a}
    \end{subfigure}
    \hfill
    \begin{subfigure}{0.32\textwidth}
        \centering
        \includegraphics[width=\linewidth]{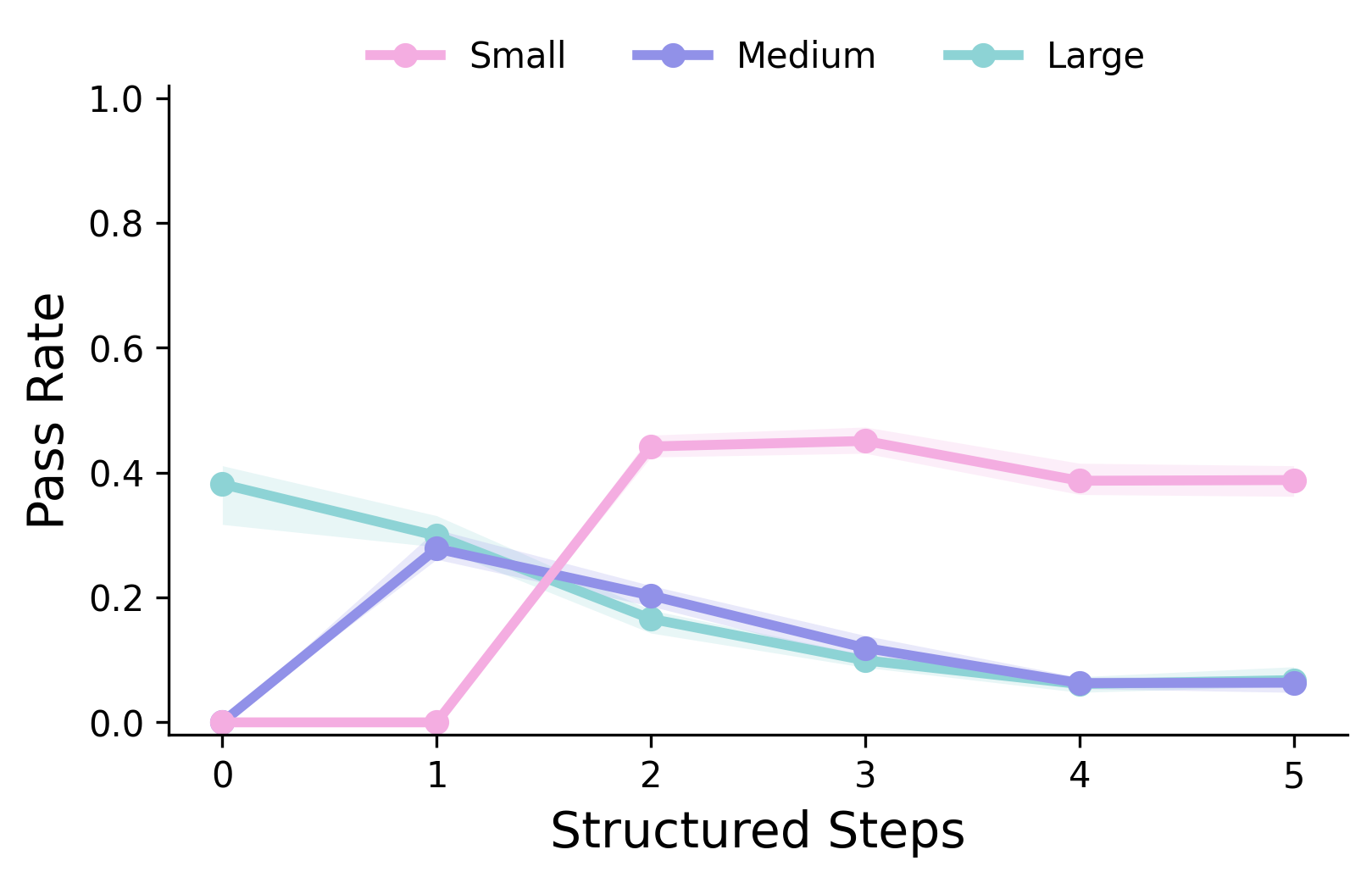}
        \caption{Pass rate under a 5 chunk candidate decomposition}
        \label{fig:2c}
    \end{subfigure}
    \hfill
    \begin{subfigure}{0.32\textwidth}
        \centering
        \includegraphics[width=\linewidth]{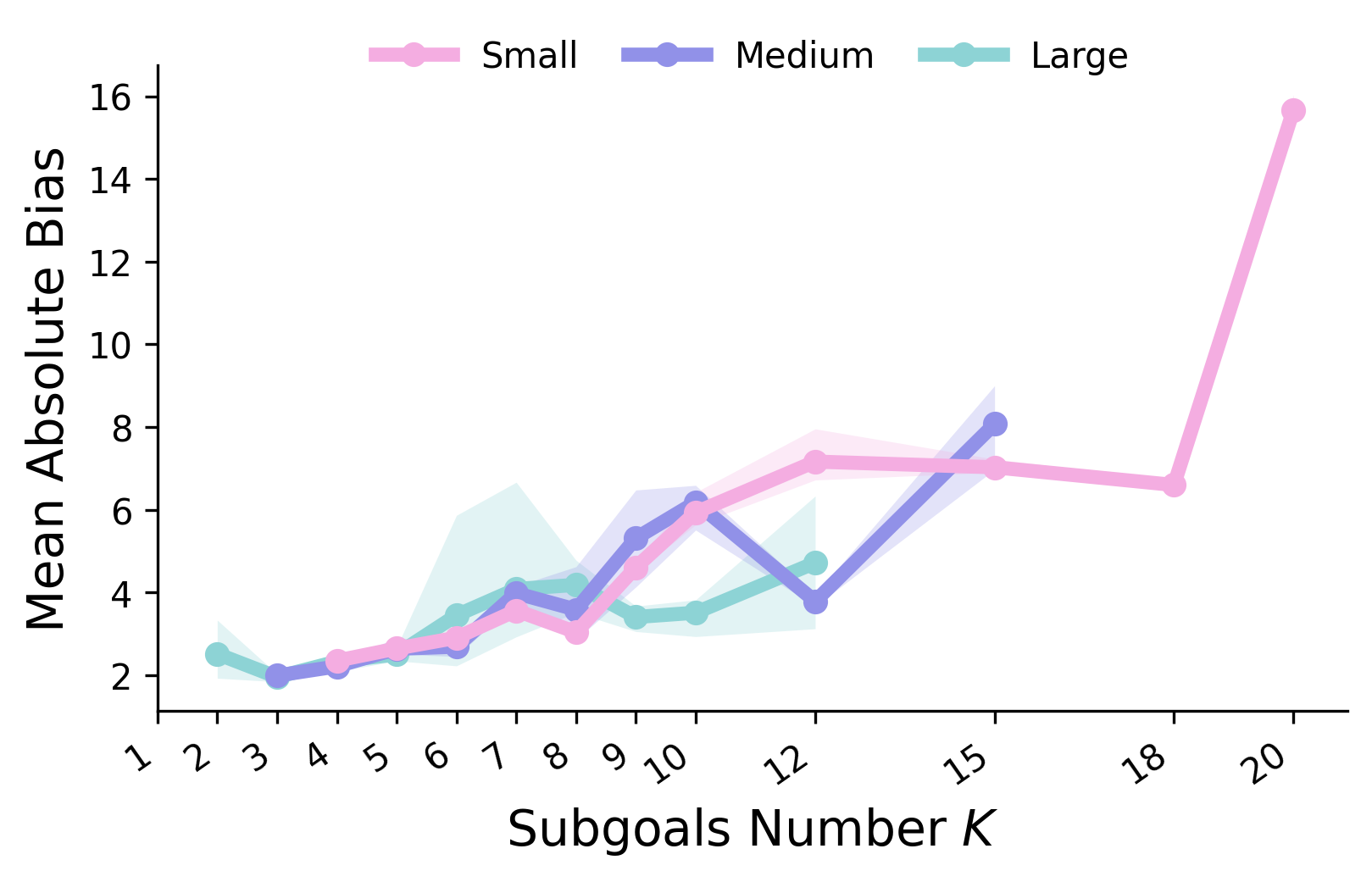}
        \caption{Final bias under different granularities}
        \label{fig:exp1_b}
    \end{subfigure}
    \hfill
    \begin{subfigure}{0.32\textwidth}
        \centering
        \includegraphics[width=\linewidth]{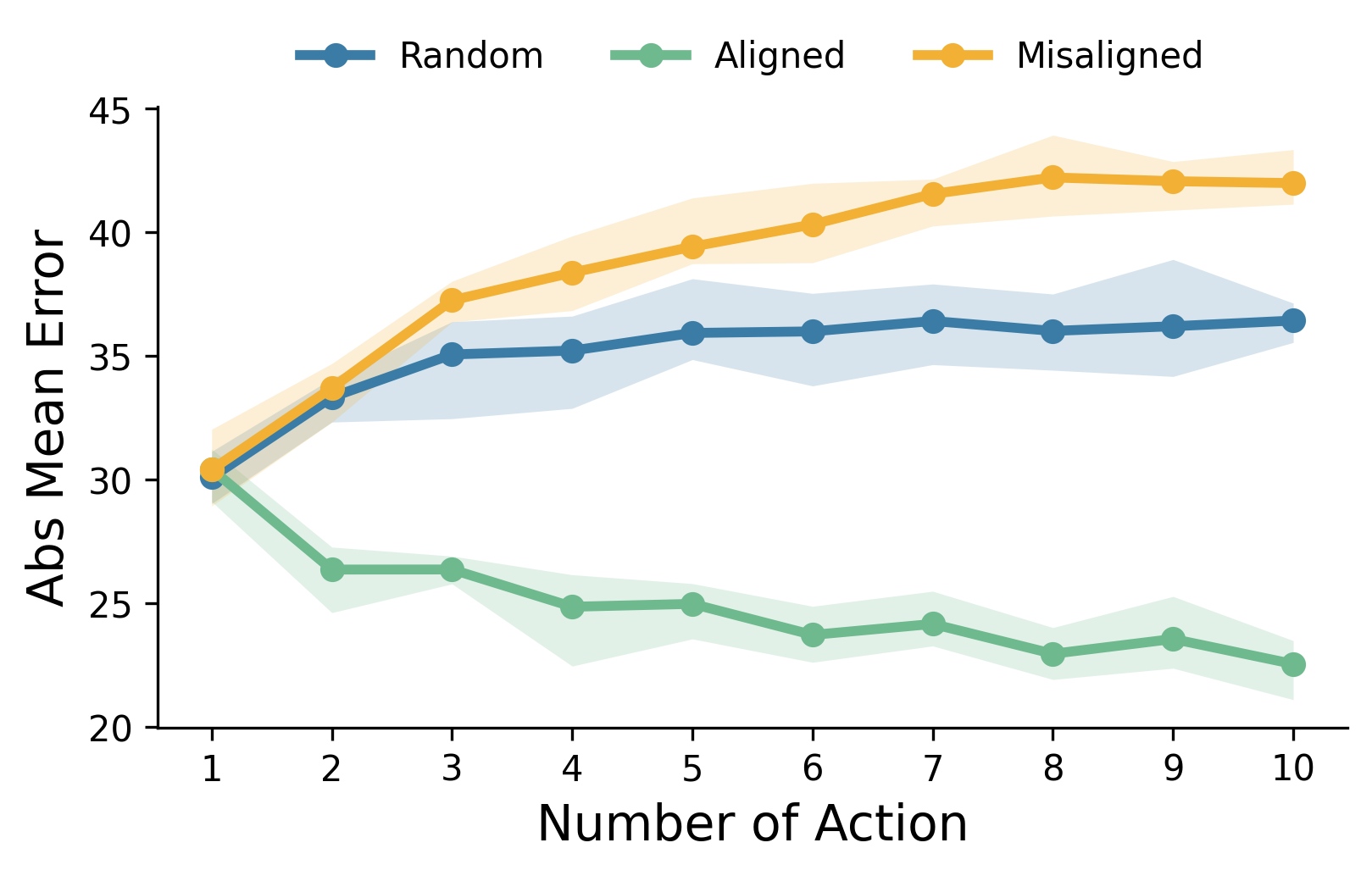}
        \caption{Final error under different guidance levels}
        \label{fig:exp2_b}
    \end{subfigure}
    \hfill
    \begin{subfigure}{0.32\textwidth}
        \centering
        \includegraphics[width=\linewidth]{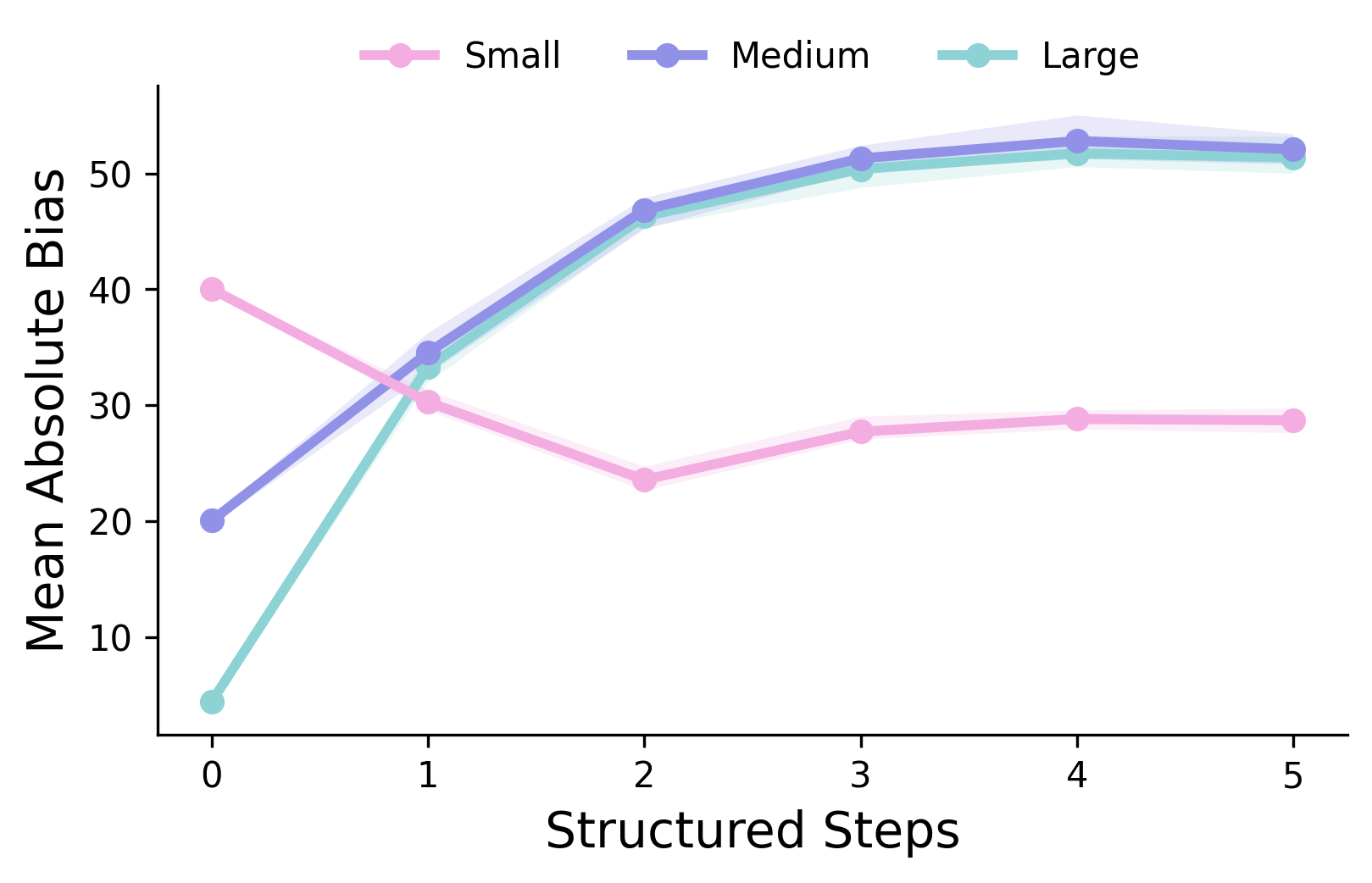}
        \caption{Final bias under a 5 chunk candidate decomposition}
        \label{fig:2f}
    \end{subfigure}
    \caption{\textbf{Three alignment principles for harness design.}
    (a, d) Granularity--capability: pass rate is non-monotonic in subgoal count $K$ and final bias grows with finer decomposition, with each agent peaking at a different $K$.
    (b, e) Guidance--evidence: aligned guidance improves pass rate and lowers final error as the action pool grows, while misaligned guidance does the opposite.
    (c, f) Partial harnessing: pass rate is unimodal in scaffold count and final bias rises beyond the peak, with stronger agents reaching the peak earlier.}
    \label{fig:sim_granularity}
\end{figure}

\textbf{Setup.}
We study a synthetic addition task in which an agent must reach a target sum $G = 100$. A \emph{harness} $\Delta_h$ decomposes $G$ into an ordered sequence of stage goals $(g_1, \dots, g_K)$ satisfying $\sum_{k=1}^{K} g_k = G$; for instance, $\Delta_h = (25, 25, 25, 25)$ when $K = 4$. Beyond decomposition, the harness also provides \emph{guidance} that shapes how the agent acts within each stage. Together, these two roles let the harness serve as a high-level planner that both issues subgoals and constrains the low-level action model used to realize them.

Within stage $k$, let $s_t = \sum_{\tau=1}^{t} a_\tau$ denote the progress accumulated after $t$ draws (reset to zero at the start of each stage). The agent maintains a pool of candidate truncated Gaussian action distributions
$$
\mathcal{P}=\left\{\mathcal{TN}_{[\ell_j,u_j]}(\mu_j,\sigma_j^2)\right\}_{j\in\mathcal J},
$$
where $\mathcal{TN}_{[\ell, u]}(\mu, \sigma^2)$ denotes a Gaussian with mean $\mu$ and variance $\sigma^2$ truncated to $[\ell, u]$. The harness's guidance is operationalized in our simplified formulation as a \emph{pruning} of this pool, leaving an admissible index set $\mathcal{J}_{\Delta_h} \subseteq \mathcal{J}$. At each draw, the agent \emph{greedily} selects $j^\star \in \mathcal{J}_{\Delta_h}$ best matched to the residual gap $g_k - s_{t-1}$, samples
$$
\tilde a \sim \mathcal{TN}_{[\ell_{j^\star},u_{j^\star}]} (\mu_{j^\star},\sigma_{j^\star}^2), \qquad a=\operatorname{round}(\tilde a)\in\mathbb Z.$$
and updates $s_t = s_{t-1} + a$. A stage is deemed successful once $s_t \in [g_k - \epsilon,\, g_k + \epsilon]$ with tolerance $\epsilon = 2$, after which the harness advances to stage $k+1$.

Based on the synthetic task above, we conduct three controlled experiments on harness execution. First, we study decomposition granularity by varying the number of subgoals $K$ across agents with different action scales. Second, we study guidance quality by varying the available action-pool size under different guidance strategies. In both experiments, we measure pass rate and final bias.
We further include additional control experiments on retry budget, completion tolerance, and guidance pruning in Appendices~\ref{app:granularity-experiment}, \ref{app:guidance-experiment}, and~\ref{app:partial-harness-experiment}..

\begin{figure}[!t]
    \centering
    \begin{subfigure}{0.47\textwidth}
        \centering
        \includegraphics[width=\linewidth]{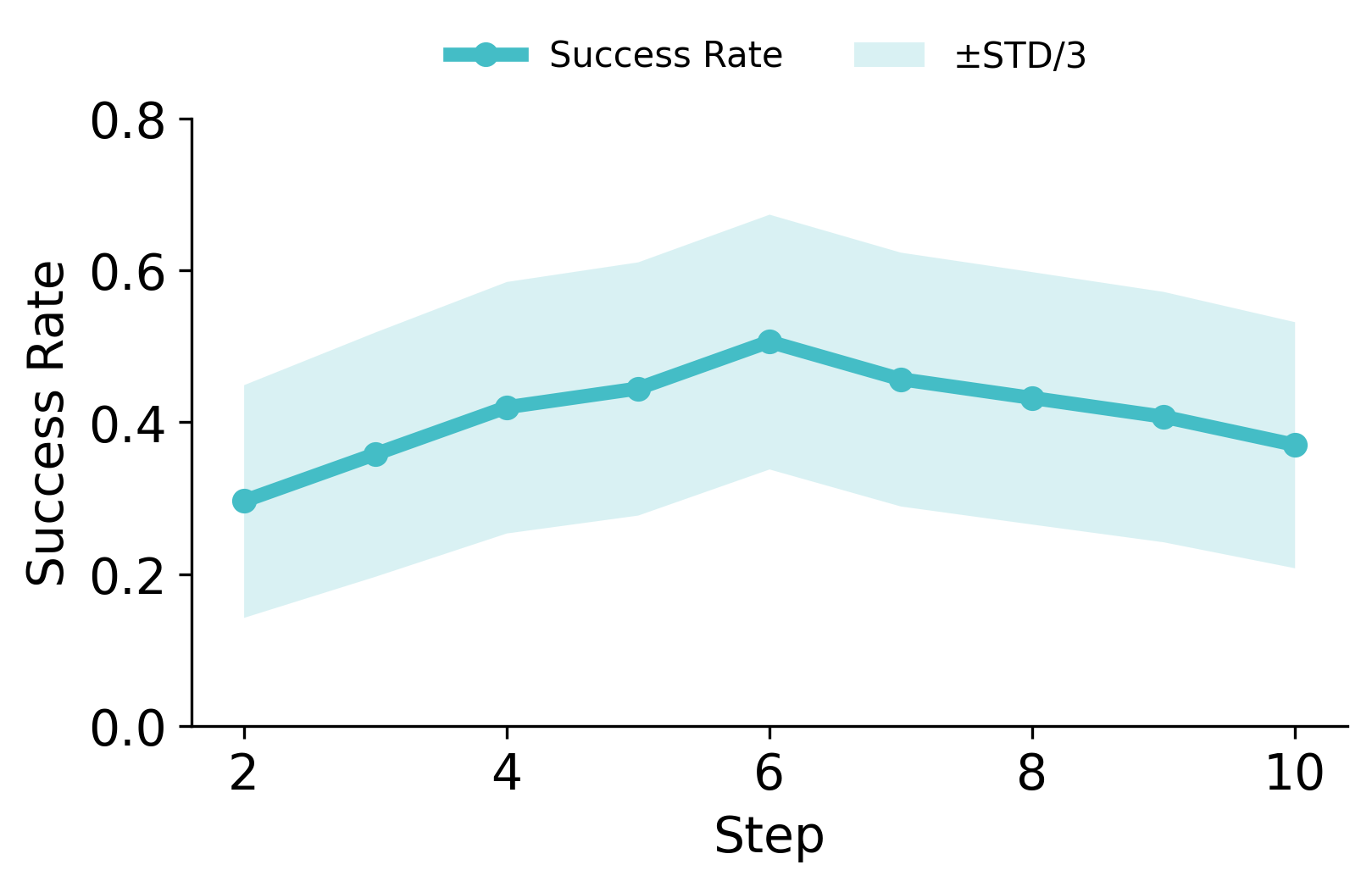}
        \caption{Workflow Granularity and Pass Rate on Terminal-Bench v2}
        \label{fig:real_step_trial1_replay}
    \end{subfigure}
    \hfill
    \begin{subfigure}{0.47\textwidth}
        \centering
        \includegraphics[width=\linewidth]{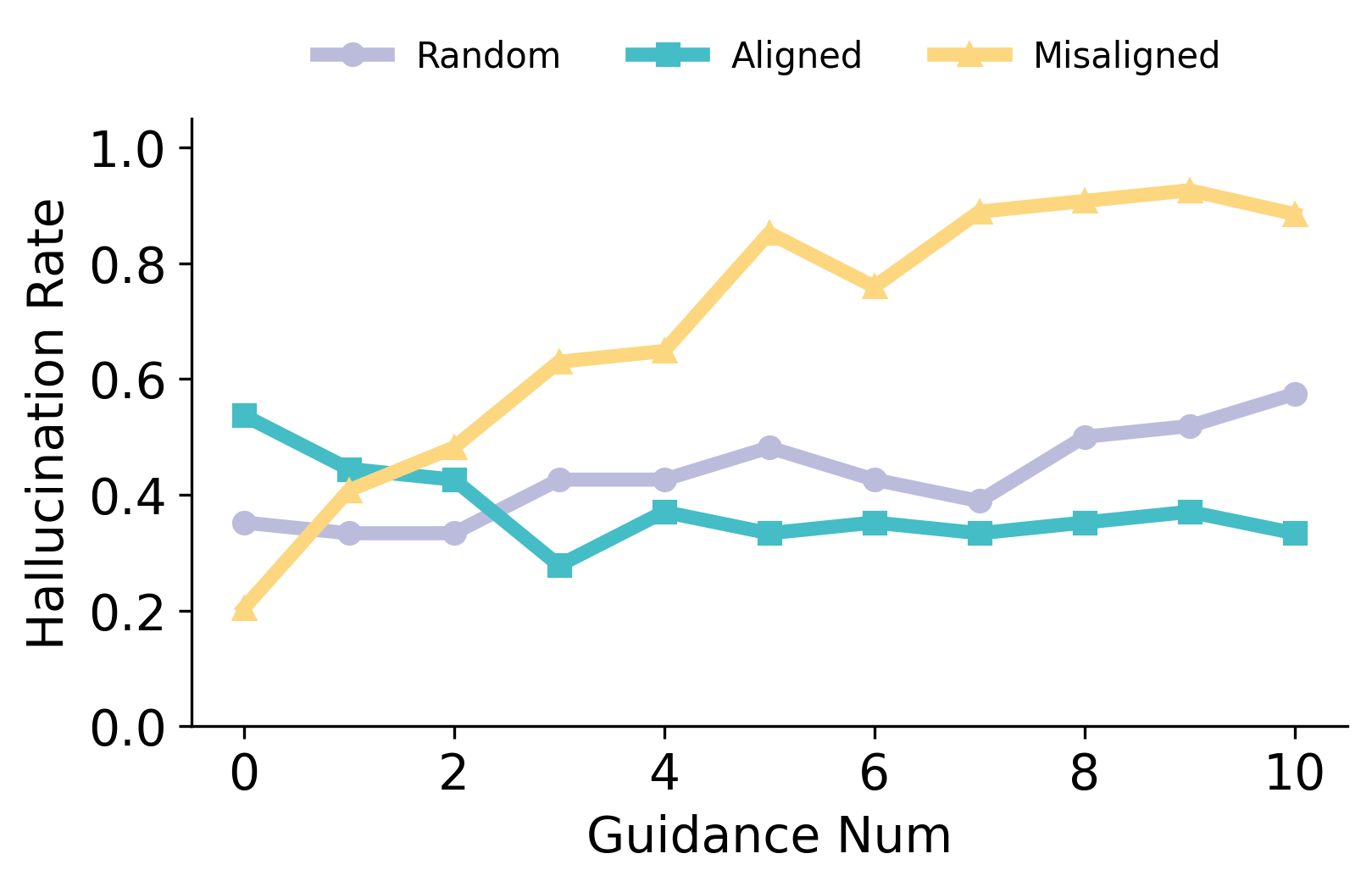}
        \caption{Guidance Quantity and Hallucination Rate under Different Alignment Conditions}
        \label{fig:real_step_trial123}
    \end{subfigure}
    \caption{
    Harness design trade-offs in real and controlled settings.
    (a) On Terminal-Bench v2, pass rate first improves and then declines as the workflow is split into more steps.
    (b) Hallucination rate under increasing amounts of guidance, comparing random, aligned, and misaligned guidance.
    }
    \label{fig:real_step_granularity}
\end{figure}

\textbf{I. Granularity--Capability Alignment.}
We vary the number of subgoals $K$ across three agents with different action scales, sweeping $G/K$ from values too coarse to be reached within the draw budget to values below the agent's natural action scale. Figures~\ref{fig:exp1_a} and~\ref{fig:exp1_b} show that pass rate peaks only when $G/K$ matches the agent's controllable progress, and that final bias grows under overly fine decomposition as local completion errors accumulate. Both effects match the granularity penalty $(\rho^{(M)}-\log M)_+$ in Theorem~\ref{thm:granularity_capability_alignment_informal}.

\textbf{II. Guidance--Evidence Alignment.}
We hold decomposition fixed and vary which action distributions remain available, comparing aligned guidance, misaligned guidance, and uniform random selection on the same pool. Figures~\ref{fig:exp2_a} and~\ref{fig:exp2_b} show that pass rate rises with pool size only under aligned guidance, while misaligned guidance produces lower pass rate and larger final error than uniform selection on the same pool. The benefit therefore comes from the sign of the retention gap rather than from additional constraints, matching Theorem~\ref{thm:guidance_retention_gap}.

\textbf{III. Partial Harnessing.}
We fix a chunk size and sweep the number of scaffolded chunks $r$, leaving the remaining length to the autonomous agent. Figures~\ref{fig:2c} and~\ref{fig:2f} show that pass rate is unimodal in $r$ for every agent, with the peak at smaller $r$ for stronger agents, while final bias rises beyond the peak as additional structure accumulates terminal error rather than reducing it. The same shape persists across alternative chunk sizes (Figure~\ref{fig:partial-extended}). The unimodal shape and agent-dependent peak location match the marginal stopping condition $\Delta(r;M)\le c_s$ in Theorem~\ref{thm:partial-harnessing}, with stronger agents reaching this condition at smaller $r$ owing to their lower autonomous tail cost $\kappa_{\mathrm{tail}}$.

\subsection{Real Data}
\label{sec:real}

\textbf{Setup.}
We conduct three experiments, each isolating one axis of harness design. (i) To probe \emph{decomposition depth}, a granularity sweep on Terminal-Bench-2~\citep{merrill2026terminalbench} varies sub-goal count $k\!\in\!\{1,\dots,10\}$ in the workflow given to a fixed agent. (ii) To probe \emph{guidance quantity and quality}, a controlled Plotly chart-analysis task varies the number of Aligned, Misaligned, and Random guidance steps and measures the hallucination rate. (iii) To probe partial harnessing, we conduct a case study on a Terminal-Bench-2 task by decomposing the solution into $10$ steps and progressively revealing a prefix of length $\ell$. Full configurations are deferred to Appendix~\ref{app:real}.

\textbf{(a) Granularity-capability alignment on Terminal-Bench v2.}
Figure~\ref{fig:real_step_trial1_replay} shows that pass rate first rises and then declines with sub-goal count, peaking at six steps. The drop on either side reflects the two failure regimes predicted by the theory: coarse workflows leave each step too large to complete reliably, and fine workflows fragment the task into milestones the model cannot meaningfully terminate at.

\textbf{(b) Guidance-evidence alignment and hallucination.}
Figure~\ref{fig:real_step_trial123} contrasts how hallucination scales with guidance count under three policies that differ only in whether they track task evidence. Under aligned guidance, hallucination stays near $0.35$ regardless of count; under misaligned guidance it rises monotonically from $0.20$ to nearly $0.90$; random guidance, which mixes the two, interpolates in between. The same scaling axis therefore has opposite effects on reliability depending on the sign of the retention gap, with stronger guidance amplifying whichever direction this sign points, matching the sign-amplification structure predicted by Theorem~\ref{thm:guidance_retention_gap}.

\subsection{A Case Study on Partial Harnessing}
\begin{figure}[t]
    \centering
    \includegraphics[width=\linewidth]{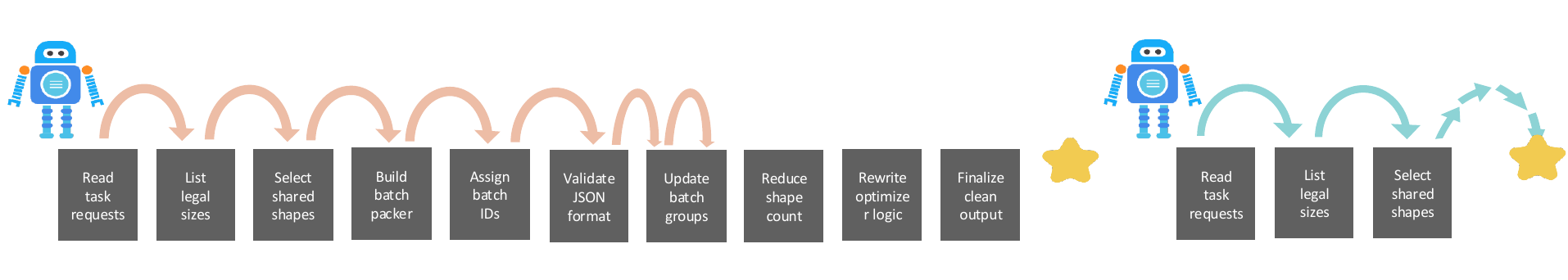}
    \caption{\textbf{Partial harnessing on llm-inference-batching-scheduler task.}
    Left: the fully specified workflow over-constrains execution and the agent gets lost in repeated intermediate revisions before reaching the final stages. Right: the partial workflow provides only an initial 3-step harness, after which the agent completes the remaining task through its own planning.}
    \label{fig:case_partial_harnessing}
\end{figure}

Beyond the empirical results above, Figure~\ref{fig:case_partial_harnessing} compares a partial harness and a full harness on the \texttt{llm-inference-batching-scheduler} task \cite{merrill2026terminalbench}. The partial harness contains only three initial steps: read the inputs, identify feasible sizes, and select shared shapes. After this initial scaffold, the agent takes over the remaining execution, enters its own iterative refinement loop, and completes the task. In contrast, the full harness specifies the downstream procedure in much greater detail, but the additional structure does not translate into better control. Instead, the agent becomes stuck inside the prescribed scaffold, repeatedly revising intermediate decisions without reaching the final solution. This case illustrates why partial harnessing can outperform full specification: a short scaffold can guide the agent into the right search space while leaving enough autonomy for it to plan, adapt, and finish the task.

%% file: text/6_conclusion.tex
\section{Conclusion}

We studied harness engineering as an inference-time alignment problem over agent execution trajectories. Separating a harness into decomposition---which sets the progress scale imposed on the agent---and guidance---which reshapes the local trajectory distribution---reveals that stronger harnesses are not necessarily better: decomposition helps only when its scale matches the agent's controllable progress, and guidance helps only when it preserves trajectories supported by the current evidence. Misalignment in either role turns scaffolding into a reliability bottleneck rather than an aid, with hallucination as one concrete instance of guidance retaining the wrong part of the local trajectory space.

This view recasts harness design from the problem of adding more structure to the problem of choosing what to specify, how strongly to specify it, and when to stop. In particular, partial harnessing shows that a harness need not cover the full execution path: once the remaining task falls within the agent's autonomous capability, continued control adds reliability cost without repaying it. The right harness is the smallest one that keeps the agent on a recoverable trajectory, and no smaller.
For clarity, we adopt simplified settings and assumptions to isolate the core mechanisms of harness design, while recognizing that real agent behavior is substantially more complex; extending these principles to richer execution dynamics and real world agent systems remains an important direction for future work.

%% file: text/999_apendix.tex
\input{text/999_2_sim_experiments}

\input{text/999_3_real_experiments}

\input{text/999_1_proof}

\input{text/999_4_prompt}

%% file: text/999_2_sim_experiments.tex
\section{Synthetic Experiments Details}
\label{app:additional_sim}

\subsection{Task and Metrics}
\label{app:synthetic-setup}

The cumulative-progress task is parameterized by a total target $G$, a tolerance $\epsilon$, and a per-stage draw budget $R$. A harness decomposes $G$ into an ordered subgoal sequence $(g_1,\ldots,g_K)$ with $\sum_k g_k=G$ and provides the agent with a pool of action distributions, each a truncated Gaussian specified by $(\mu,\sigma,\ell,u)$. At stage $k$, local progress $s_k$ is initialized to zero. The agent draws $z\sim\mathcal{N}(\mu,\sigma^2)$, rejects samples outside $[\ell,u]$, rounds to the nearest integer, and clips back into $[\ell,u]$; the resulting action $a$ updates $s_k\leftarrow s_k+a$. A stage succeeds once $s_k\in[g_k-\epsilon,g_k+\epsilon]$ after at least one draw, fails by overshoot if $s_k>g_k+\epsilon$, and fails by draw limit if $R$ draws are exhausted before the tolerance window is reached. An episode succeeds if and only if all stages succeed.

We report two metrics throughout. \emph{Pass rate} is the fraction of episodes in which all stages succeed and measures whether the harness keeps the run on a recoverable trajectory. \emph{Absolute final bias} is $|\sum_k s_k - G|$ and measures how much terminal accuracy the harness preserves even when the run completes. The two metrics can dissociate, and reporting both separates "the harness lets the agent finish" from "the harness lets the agent finish accurately."

\subsection{Granularity--Capability Experiment}
\label{app:granularity-experiment}

This experiment isolates decomposition by varying $K$ while fixing a single action distribution per agent. We use $G=100$, $\epsilon=2$, $R=4$, and sweep $K\in\{1,2,3,4,5,6,7,8,9,10,12,15,18,20\}$. The three agents differ only in action scale: Small $(\mu=6,\sigma=2,[4,8])$, Medium $(\mu=8,\sigma=3,[5,11])$, and Large $(\mu=10,\sigma=4,[6,14])$. Small $K$ produces large subgoals that may exceed what an agent can reach within $R$ draws, while large $K$ produces subgoals below the agent's natural action scale, where rounding-and-clipping introduces residual error at every stage. The three agents therefore peak at different $K$, sweeping out the granularity--capability frontier predicted by Theorem~\ref{thm:granularity_capability_alignment_informal}.

\subsection{Guidance--Evidence Experiment}
\label{app:guidance-experiment}

This experiment isolates guidance quality from action-pool size. We construct a pool of ten distributions indexed by $i=0,\ldots,9$, with
$$
\mu_i=4.0+1.2i,\qquad \sigma_i=1.5+0.35i,\qquad
\ell_i=\max\{1,\lfloor\mu_i-2.0\rfloor\},\qquad
u_i=\lceil\mu_i+2\sigma_i\rceil.
$$
At the start of each episode we randomly retain $N\in\{1,\ldots,10\}$ distributions, then compare three selection policies on the retained pool: aligned guidance (favoring distributions whose mean is closest to the remaining progress), misaligned guidance (favoring distributions whose mean is farthest), and uniform random selection. We use $G=100$, $K=5$, $\epsilon=4$, and $R=5$. Holding pool size fixed across the three policies isolates the sign of the retention gap from the size of the local action space, which is the comparison Theorem~\ref{thm:guidance_retention_gap} predicts to be decisive.

\subsection{Partial Harnessing Experiment}
\label{app:partial-harness-experiment}

This experiment sweeps a progress slice as in Theorem~\ref{thm:partial-harnessing}. We fix a chunk size $c$ and a number of scaffolded chunks $r$, defining the harness
$$
\Delta_h=(\underbrace{c,\ldots,c}_{r\text{ chunks}},\,G-rc),
$$
where $r=0$ leaves the entire task to the autonomous agent and $r=\lfloor G/c\rfloor$ scaffolds it fully. We use $G=100$, $\epsilon=2$, $R=10$, and the Small, Medium, and Large agents from Appendix~\ref{app:granularity-experiment}. Sweeping $r$ traces the slice objective $F(r)$ for each agent. A stronger agent has lower autonomous tail cost $\kappa_{\mathrm{tail}}$ and reaches $\Delta(r;M)\le c_s$ at smaller $r$, predicting earlier peaks for stronger agents.

\subsection{Additional Control Experiments}
\label{app:additional-experiments}

The experiments below vary harness or task parameters that could plausibly explain the patterns above through a simpler mechanism. In each case the result reproduces the alignment predictions rather than displacing them.

\paragraph{Retry budget.}
Holding $K=4$ and the granularity setup of Appendix~\ref{app:granularity-experiment} fixed, we sweep the per-stage draw budget $R\in\{1,\ldots,10\}$. Increasing $R$ raises pass rate and lowers final bias, but the curves saturate quickly once the binding failure mode shifts from draw-limit exhaustion to action-scale mismatch and overshoot (Figure~\ref{fig:attempt}). Retry budget is therefore a recoverability resource that interacts with granularity rather than substituting for it: more attempts cannot fix a stage whose requested progress lies outside the agent's reachable scales, matching the $-\log M_t$ correction in Theorem~\ref{thm:granularity_capability_alignment_informal}.

\begin{figure}[t]
    \centering
    \begin{subfigure}{0.47\textwidth}\centering
        \includegraphics[width=\linewidth]{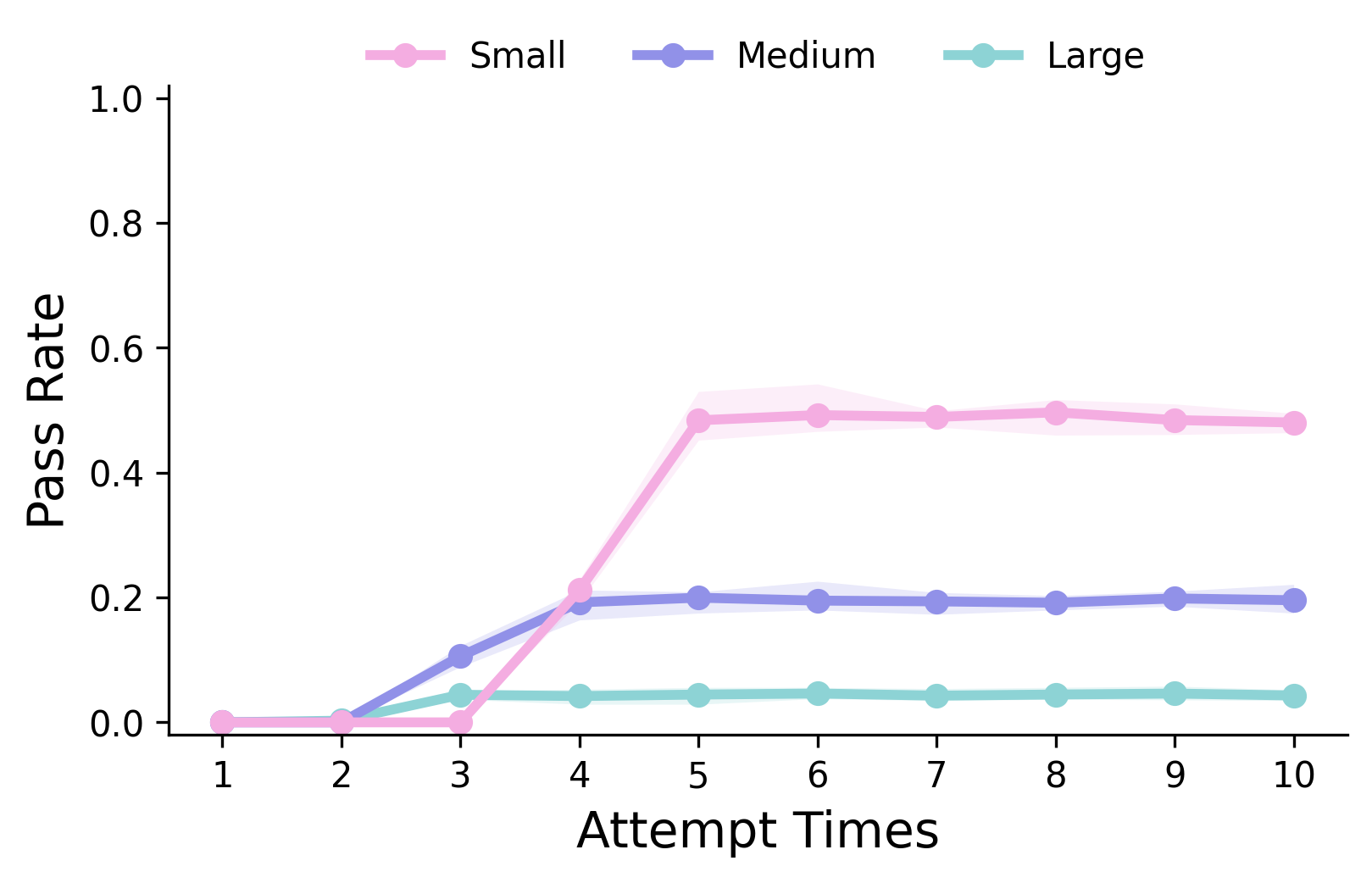}
        \caption{Pass rate vs.\ attempt budget}\label{fig:attempt_pass_rate}
    \end{subfigure}\hfill
    \begin{subfigure}{0.47\textwidth}\centering
        \includegraphics[width=\linewidth]{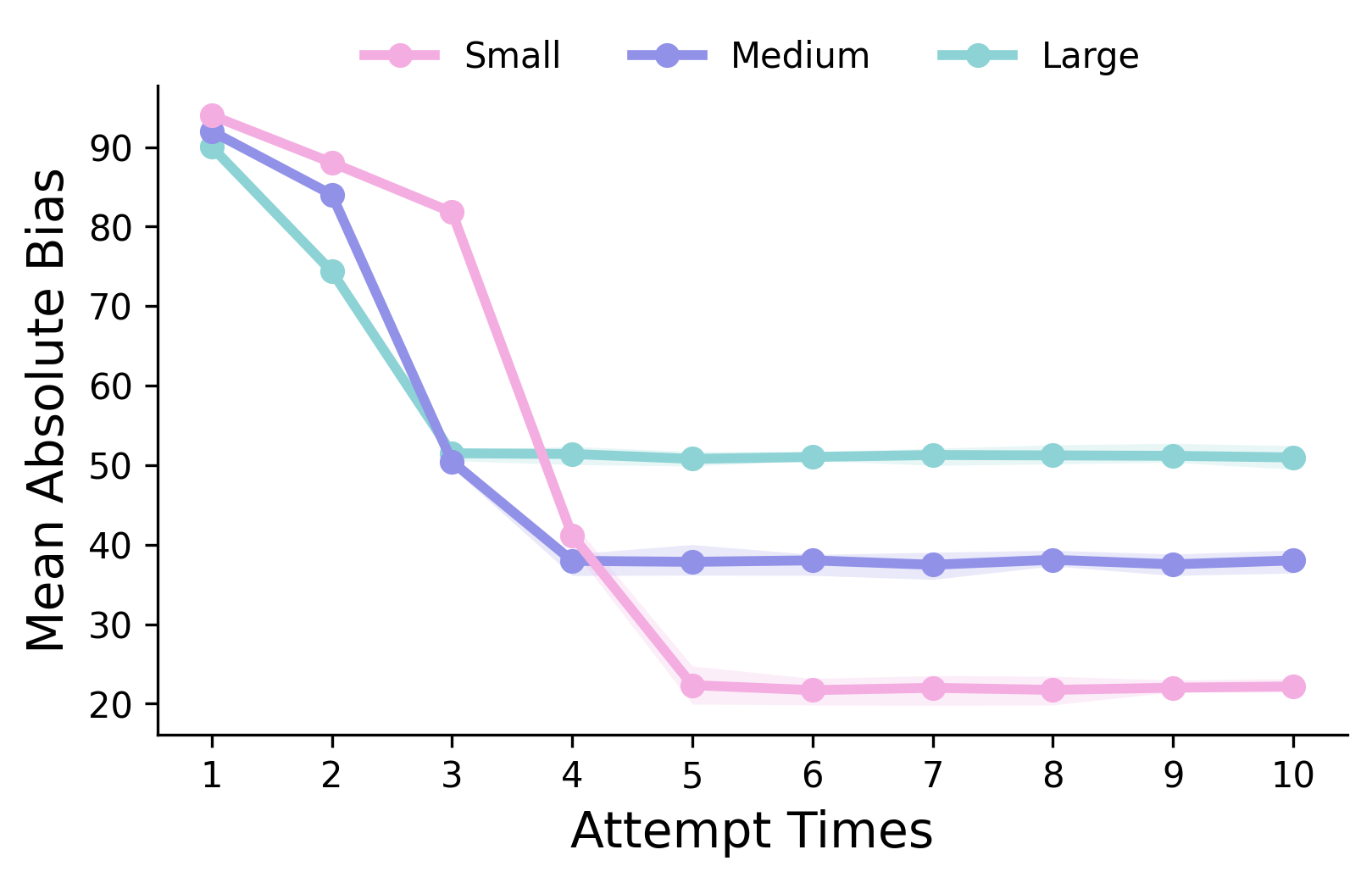}
        \caption{Final bias vs.\ attempt budget}\label{fig:attempt_bias}
    \end{subfigure}
    \caption{\textbf{Retry budget improves recoverability but cannot overcome capability mismatch.} Pass rate and final bias both saturate once attempts no longer relax the binding constraint.}
    \label{fig:attempt}
\end{figure}

\paragraph{Completion tolerance.}
Fixing $G=100$, $K=10$, and $R=4$, we sweep $\epsilon\in\{0,1,\ldots,10\}$. Larger tolerance raises pass rate by widening each stage's acceptance window, but for agents whose action scale undershoots the subgoal size it accepts premature stage completion and accumulates systematic terminal under-progress (Figure~\ref{fig:tolerance}). Tolerance therefore trades recoverability against terminal accuracy, and the right setting depends on the agent's action scale rather than functioning as a universal robustness knob.

\begin{figure}[t]
    \centering
    \begin{subfigure}{0.47\textwidth}\centering
        \includegraphics[width=\linewidth]{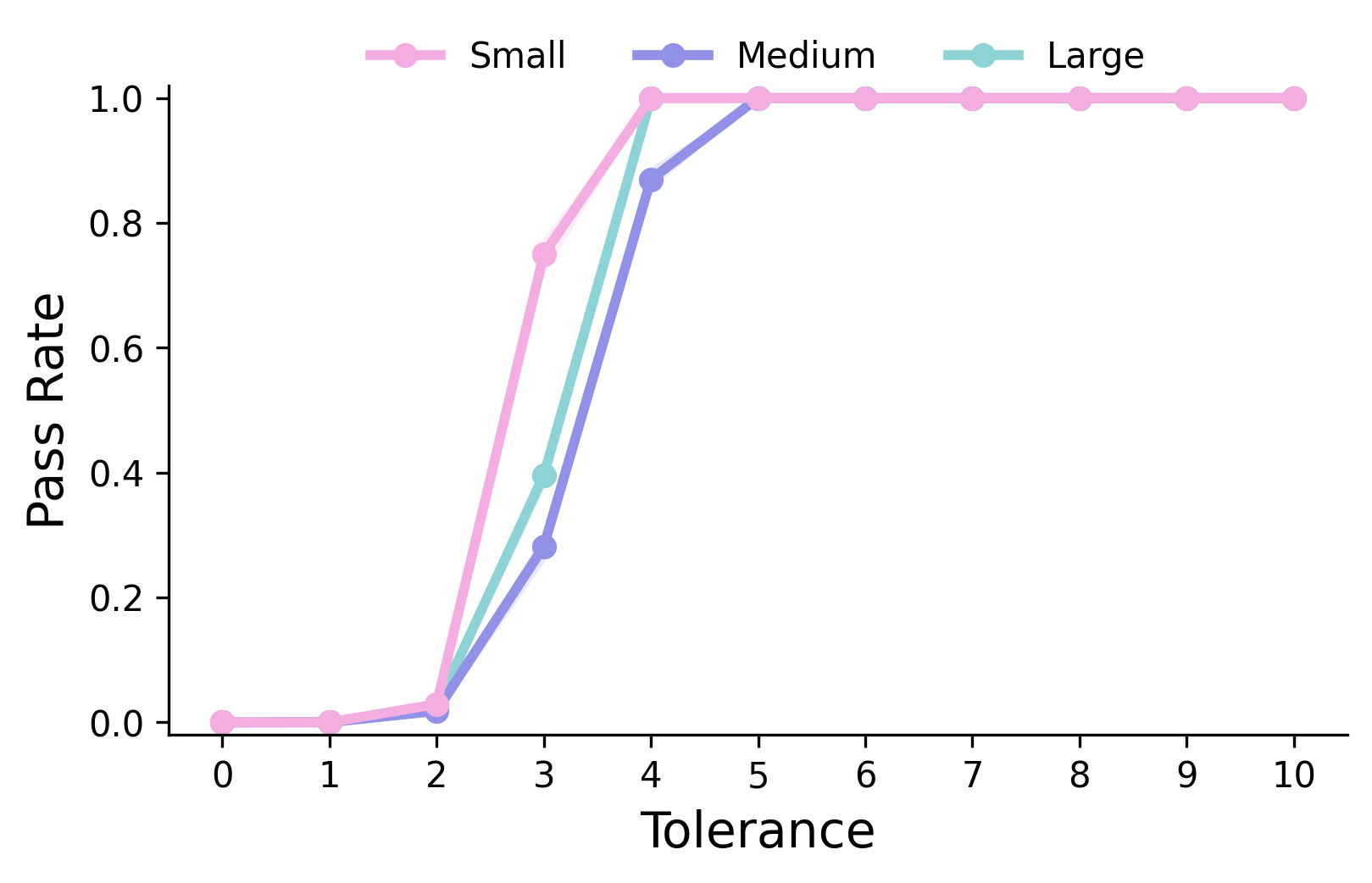}
        \caption{Pass rate vs.\ tolerance}\label{fig:tolerance_pass_rate}
    \end{subfigure}\hfill
    \begin{subfigure}{0.47\textwidth}\centering
        \includegraphics[width=\linewidth]{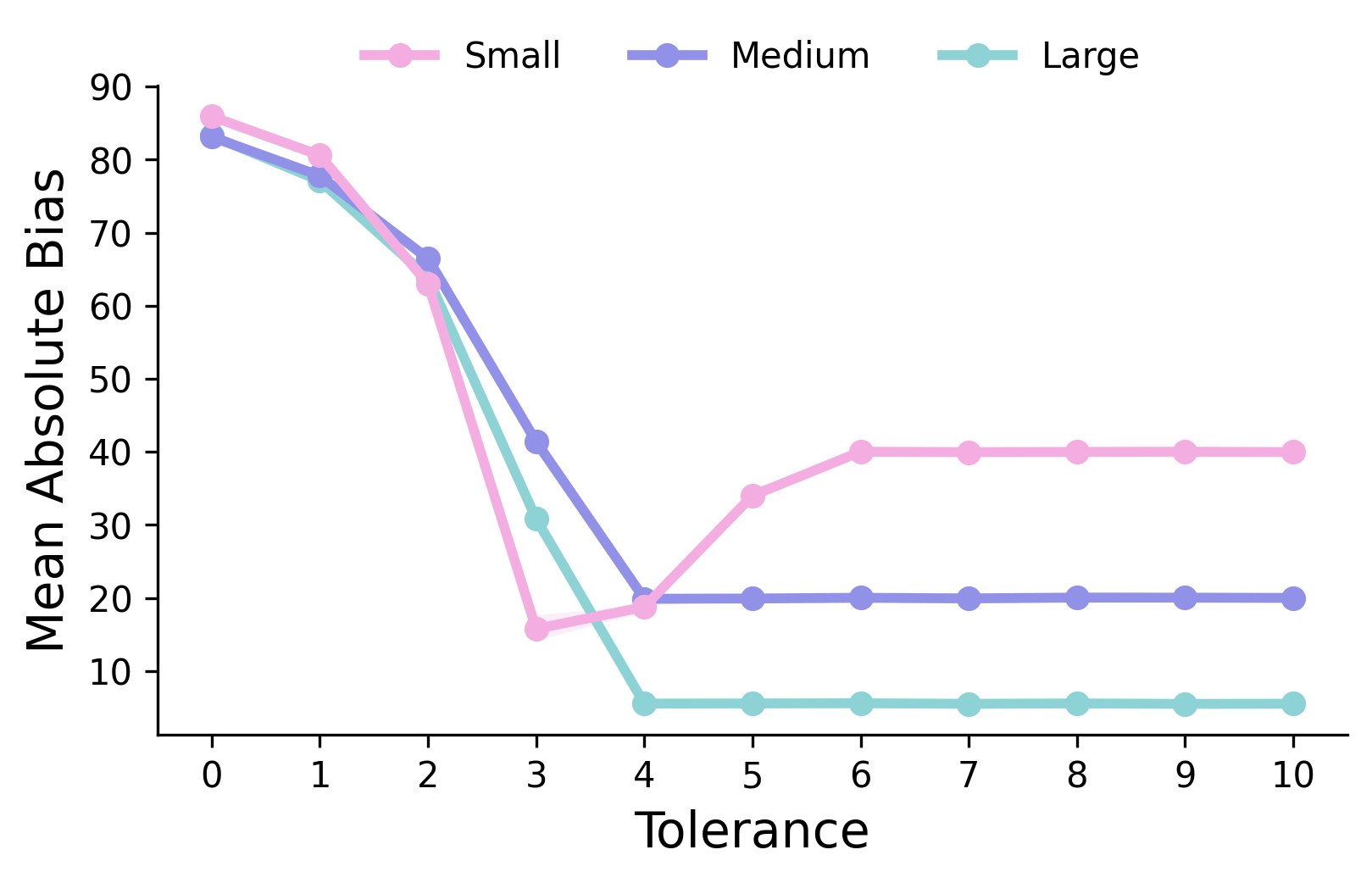}
        \caption{Final bias vs.\ tolerance}\label{fig:tolerance_bias}
    \end{subfigure}
    \caption{\textbf{Tolerance trades recoverability for terminal accuracy.} Pass rate rises sharply with $\epsilon$ but final bias grows when tolerance exceeds the agent's natural per-stage variation.}
    \label{fig:tolerance}
\end{figure}

\paragraph{Aggressive guidance pruning.}
Starting from a three-distribution pool---Model 1 $(\mu=6,\sigma=2,[4,8])$, Model 2 $(\mu=8,\sigma=3,[5,11])$, Model 3 $(\mu=10,\sigma=6,[4,14])$---we randomly remove $m\in\{0,1,2\}$ distributions and sweep $K$ as in Appendix~\ref{app:granularity-experiment}, with the agent greedily selecting the closest-mean distribution at each step. Stronger pruning lowers pass rate and amplifies final bias, with the largest gap appearing under fine decomposition where local action-scale mismatch is most exposed (Figure~\ref{fig:hard_guidance}). Pruning is therefore a negative-gap intervention even when individual remaining distributions are reasonable: the retention-gap sign in Theorem~\ref{thm:guidance_retention_gap} depends on which trajectories are removed, not how many.

\begin{figure}[t]
    \centering
    \begin{subfigure}{0.47\textwidth}\centering
        \includegraphics[width=\linewidth]{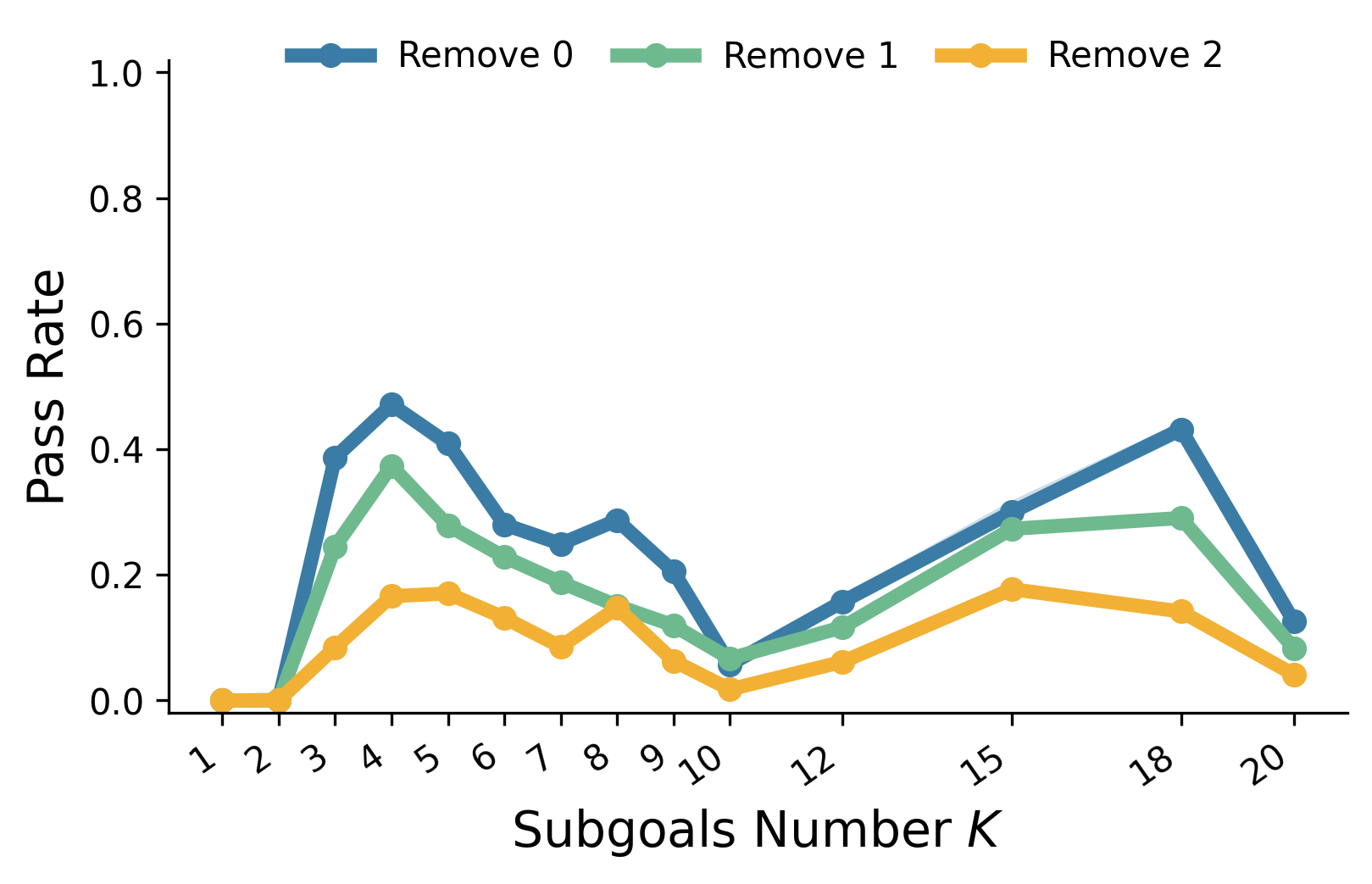}
        \caption{Pass rate vs.\ subgoal count under pruning}\label{fig:hard_guidance_pass_rate}
    \end{subfigure}\hfill
    \begin{subfigure}{0.47\textwidth}\centering
        \includegraphics[width=\linewidth]{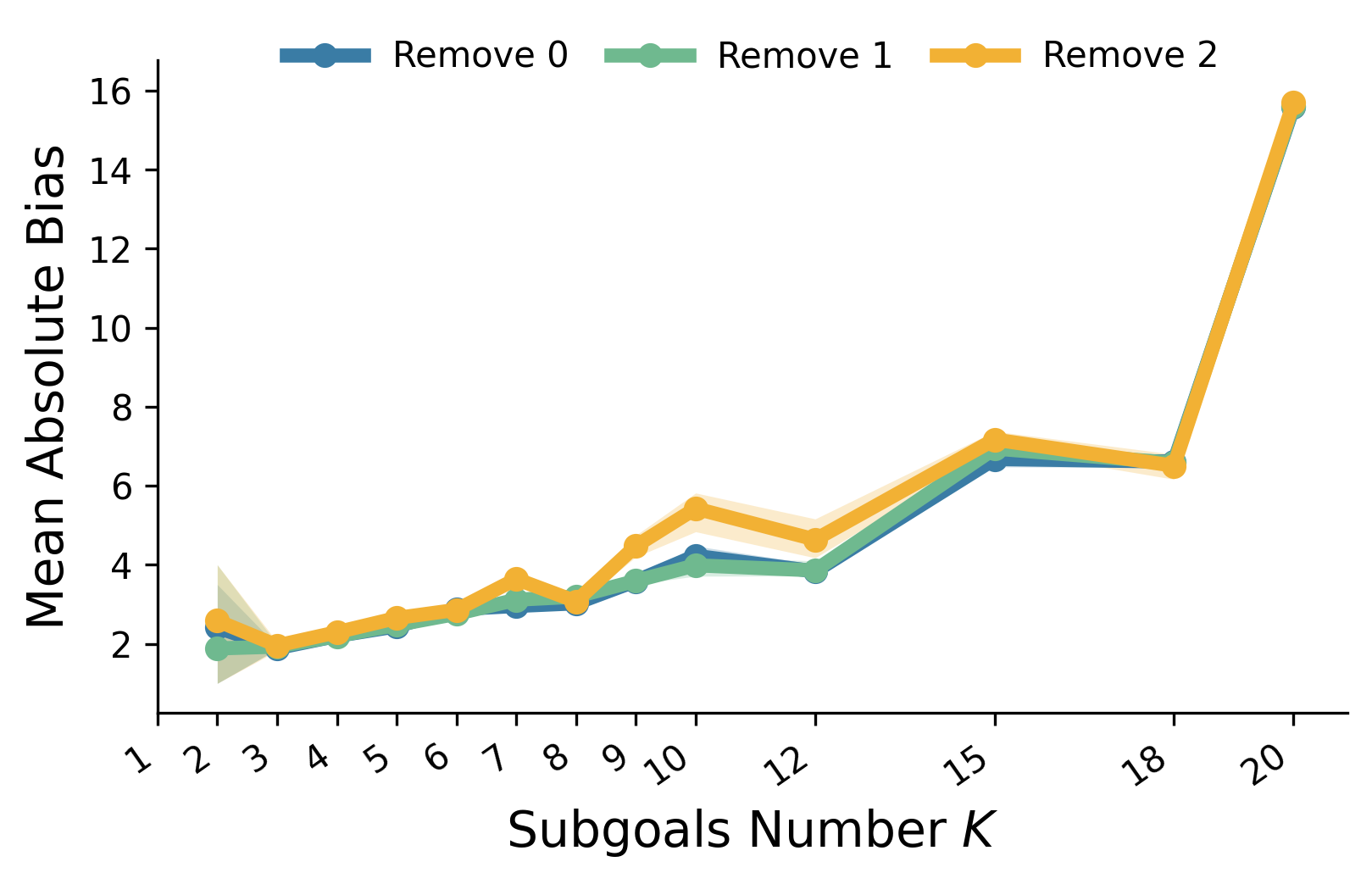}
        \caption{Final bias vs.\ subgoal count under pruning}\label{fig:hard_guidance_bias}
    \end{subfigure}
    \caption{\textbf{Aggressive pruning amplifies granularity mismatch.} Removing more action distributions worsens both pass rate and final bias, with the largest gap appearing under fine decomposition where local mismatch is most exposed.}
    \label{fig:hard_guidance}
\end{figure}

\paragraph{Extended partial-harnessing sweeps.}
We complement the chunk size $c=20$ result reported in the main text with $c=10$ and $c=25$ (Figure~\ref{fig:partial-extended}). The unimodal shape and the agent-dependent peak location persist across chunk sizes, with finer chunks placing the peak at larger $r$ as predicted by the marginal rule.

\begin{figure}[t]
    \centering
    \begin{subfigure}{0.47\textwidth}\centering
        \includegraphics[width=\linewidth]{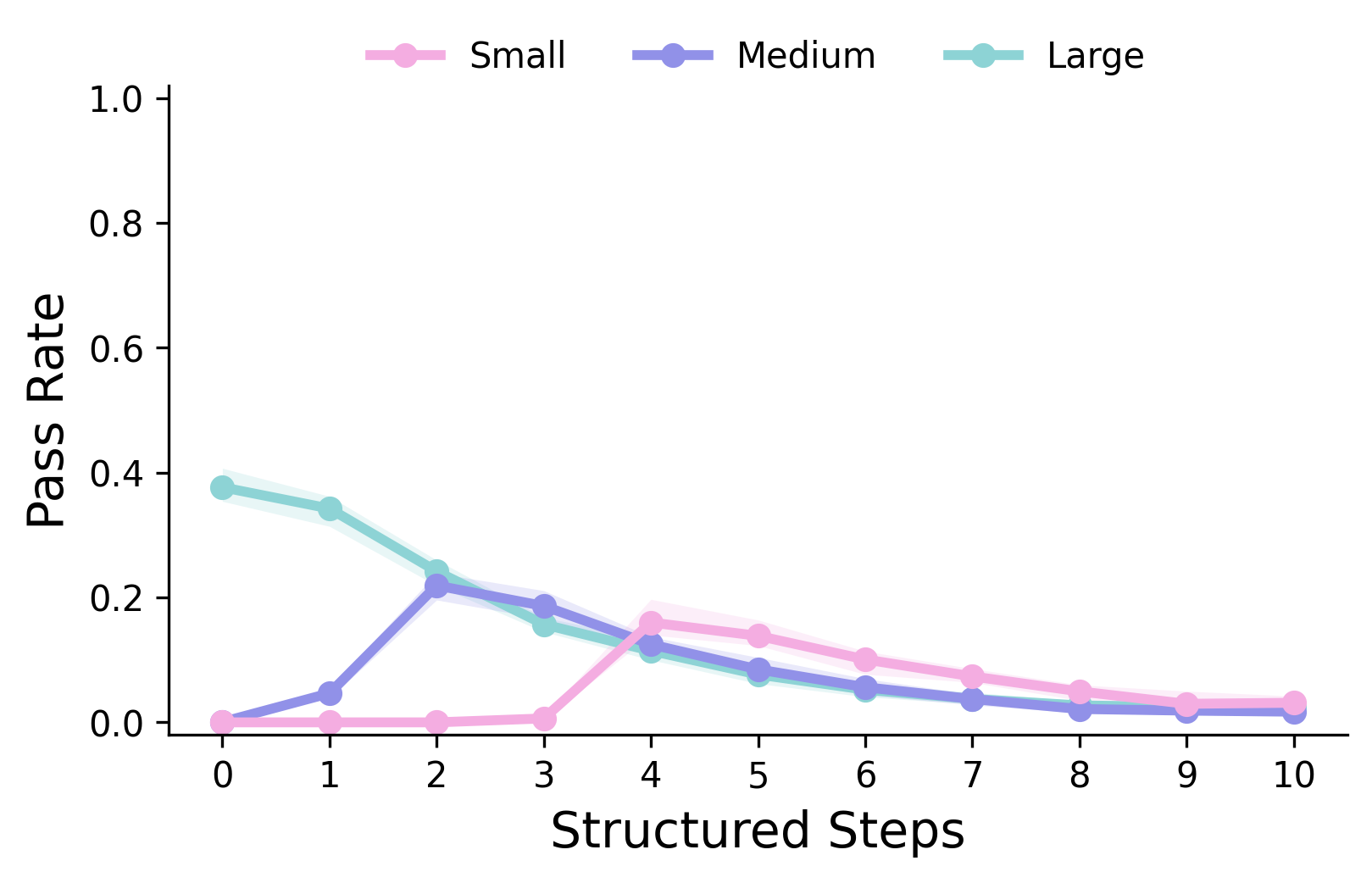}
        \caption{Pass rate, $c=10$}
    \end{subfigure}\hfill
    \begin{subfigure}{0.47\textwidth}\centering
        \includegraphics[width=\linewidth]{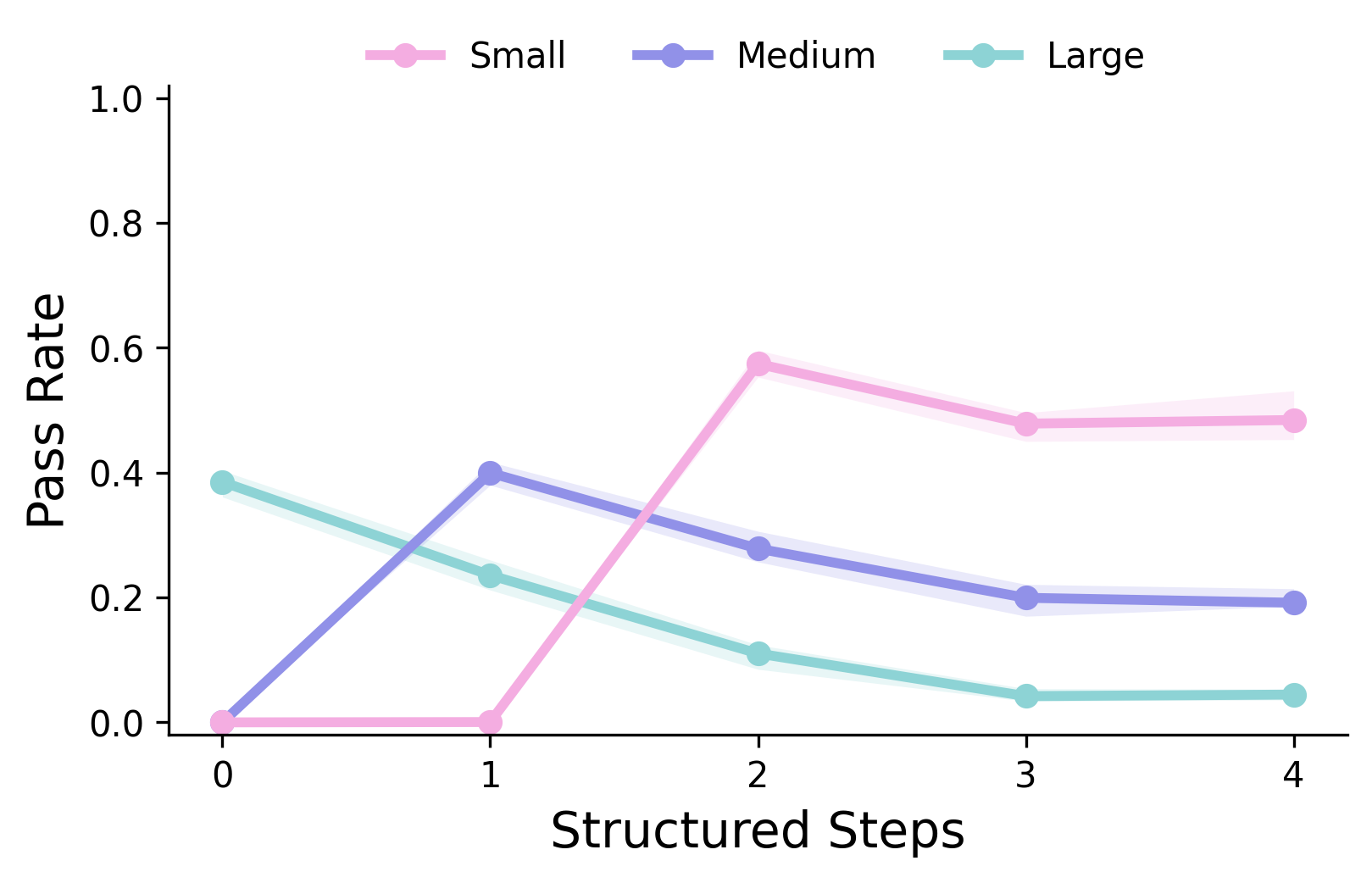}
        \caption{Pass rate, $c=25$}
    \end{subfigure}
    \begin{subfigure}{0.47\textwidth}\centering
        \includegraphics[width=\linewidth]{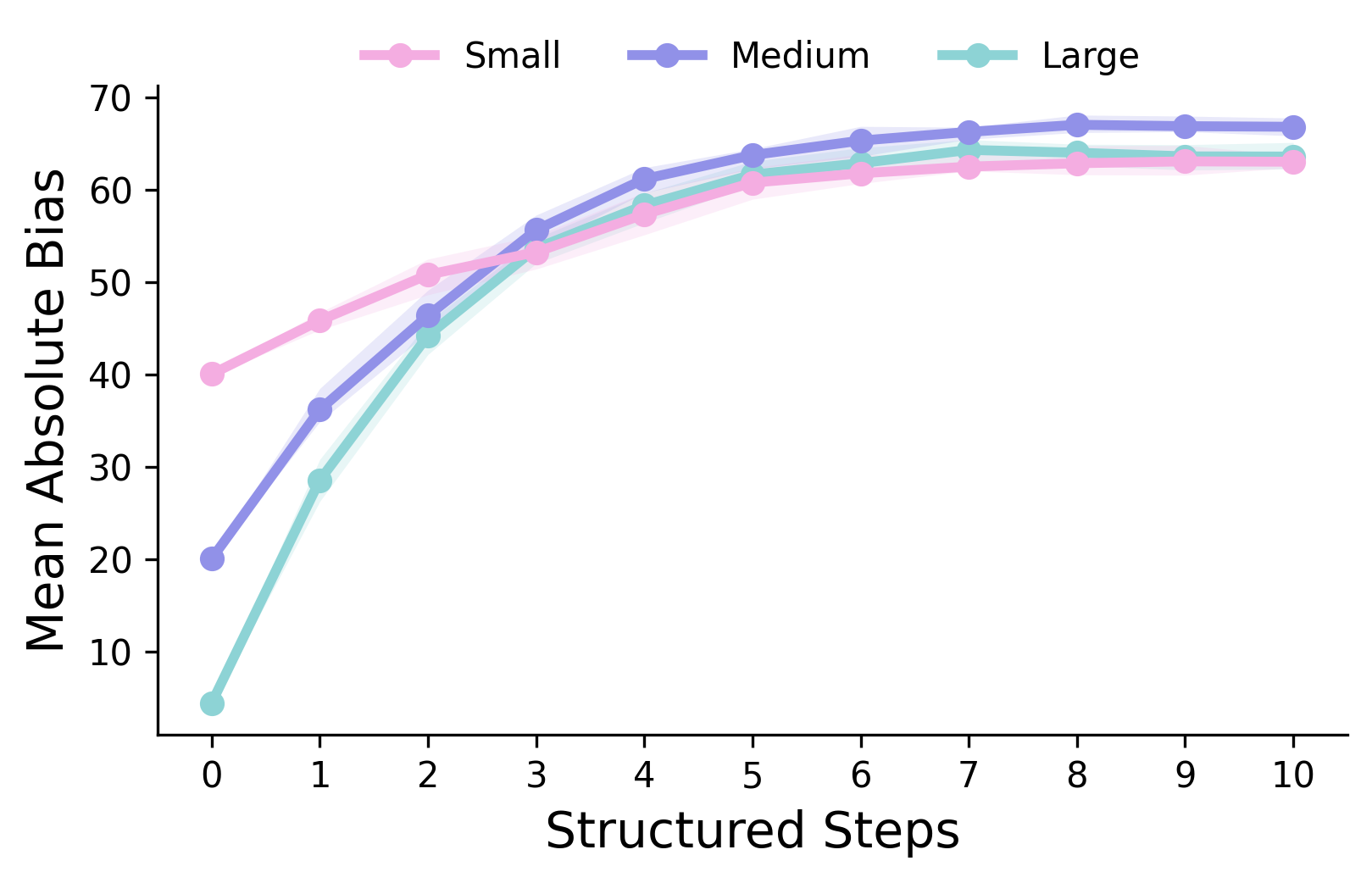}
        \caption{Final bias, $c=10$}
    \end{subfigure}\hfill
    \begin{subfigure}{0.47\textwidth}\centering
        \includegraphics[width=\linewidth]{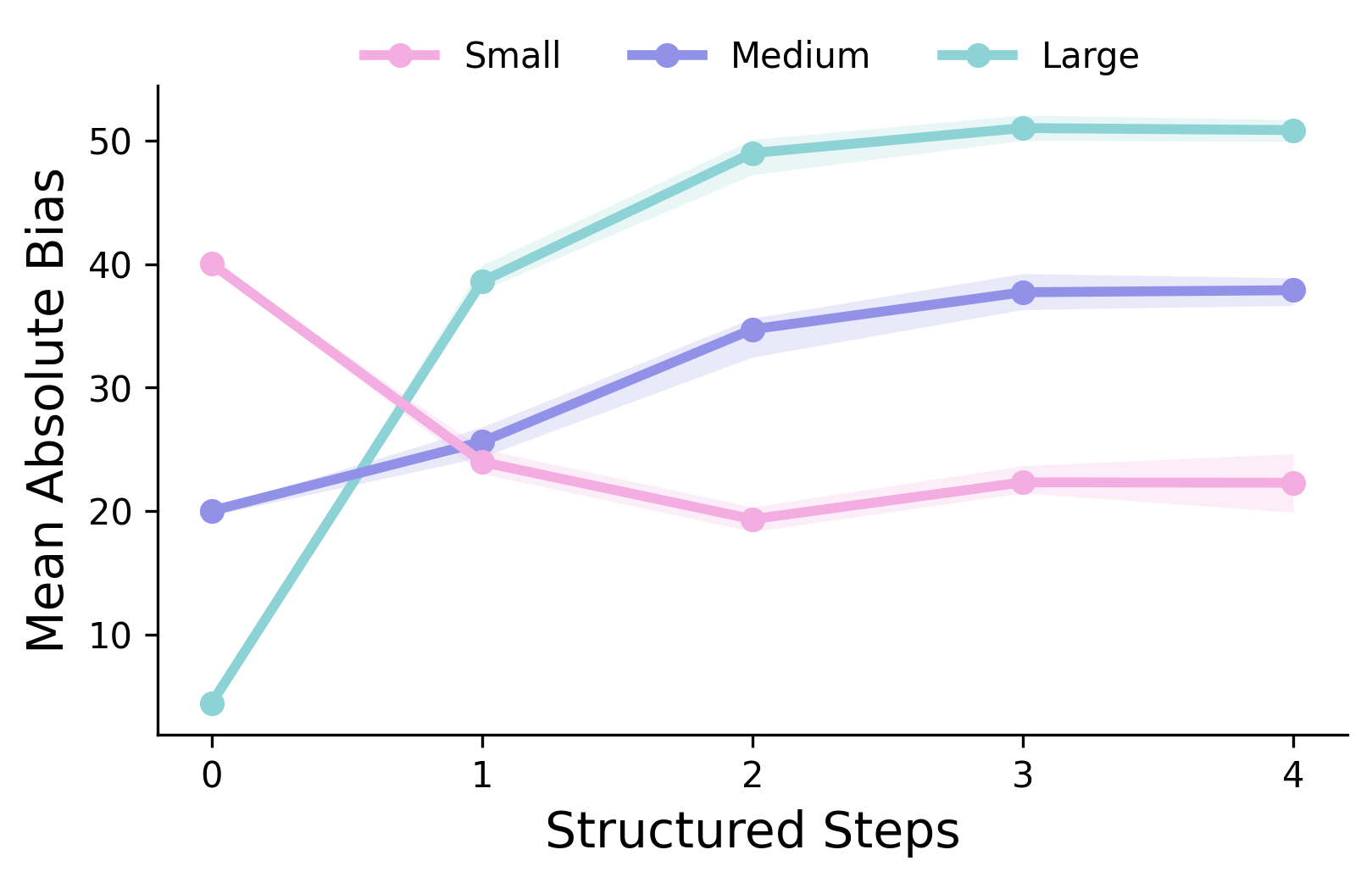}
        \caption{Final bias, $c=25$}
    \end{subfigure}
    \caption{\textbf{Marginal stopping persists across chunk sizes.} Both pass rate and final bias remain unimodal in $r$ at $c=10$ and $c=25$, with peak location shifting as predicted by Theorem~\ref{thm:partial-harnessing}.}
    \label{fig:partial-extended}
\end{figure}

%% file: text/999_3_real_experiments.tex
\section{Real World Task Details}
\label{app:real}

\subsection{Granularity experiment on Terminal-Bench-2}
\label{app:tb2_step_setup}

This appendix specifies the full configuration of the granularity sweep referenced in Section~\ref{sec:real}. The experiment varies the workflow step count $k\!\in\!\{1,\dots,10\}$ supplied to a fixed solver and measures pass rate on real software-engineering tasks from \texttt{terminal-bench@2.0}, with all non-$k$ factors held constant.

\subsubsection{Tasks}

We sample tasks from \texttt{terminal-bench@2.0} using the official Terminal-Bench leaderboard's \texttt{GLM-5} task-level correctness statistics (snapshot: May 2026). The selection rule is to focus on tasks of moderate difficulty for \texttt{GLM-5}, down-weighting both near-ceiling tasks (uninformative because already trivial) and near-floor tasks (uninformative because no workflow yields signal), so that granularity effects can be observed in the sensitive middle band. Applying this rule yields a pool of 36 tasks.

For the aggregate step-curve in the main text, we further restrict to a $32$-task subset by removing four tasks that returned all-zero outcomes in our first trial across every $k$: \texttt{break-filter-js-from-html}, \texttt{caffe-cifar-10}, \texttt{chess-best-move}, and \texttt{db-wal-recovery}. Excluding these prevents degenerate floor cases from flattening the aggregate curve.

\begin{table}[t]
\centering
\footnotesize
\caption{Selected Terminal-Bench 2 tasks and official resolution rates in \cite{merrill2026terminalbench}.}
\label{tab:tbench2_selected_tasks}
\begin{tabular}{lc|lc}
\toprule
Task & Rate & Task & Rate \\
\midrule
break-filter-js-from-html & 0.0\% &
caffe-cifar-10 & 0.0\% \\
chess-best-move & 0.0\% &
db-wal-recovery & 0.0\% \\
bn-fit-modify & 20.0\% &
build-cython-ext & 20.0\% \\
cancel-async-tasks & 20.0\% &
circuit-fibsqrt & 20.0\% \\
compile-compcert & 25.0\% &
overfull-hbox & 20.0\% \\
protein-assembly & 20.0\% &
query-optimize & 20.0\% \\
sanitize-git-repo & 20.0\% &
winning-avg-corewars & 20.0\% \\
adaptive-rejection-sampler & 40.0\% &
feal-differential-cryptanalysis & 40.0\% \\
log-summary-date-ranges & 40.0\% &
build-pov-ray & 60.0\% \\
extract-elf & 60.0\% &
qemu-startup & 60.0\% \\
sparql-university & 60.0\% &
tune-mjcf & 60.0\% \\
configure-git-webserver & 80.0\% &
count-dataset-tokens & 80.0\% \\
custom-memory-heap-crash & 80.0\% &
financial-document-processor & 80.0\% \\
fix-ocaml-gc & 80.0\% &
headless-terminal & 80.0\% \\
kv-store-grpc & 80.0\% &
large-scale-text-editing & 80.0\% \\
llm-inference-batching-scheduler & 80.0\% &
model-extraction-relu-logits & 80.0\% \\
qemu-alpine-ssh & 80.0\% &
sqlite-db-truncate & 80.0\% \\
sqlite-with-gcov & 80.0\% &
largest-eigenval & 100.0\% \\
\bottomrule
\end{tabular}
\end{table}

\subsubsection{Variables and Controls}

The independent variable is the workflow step count $k\!\in\!\{1,\dots,10\}$. The primary dependent variable is pass rate; we additionally log \texttt{episodes\_to\_success} (number of solver episodes used on successful runs) as a secondary diagnostic. To isolate the effect of $k$, the following factors are held fixed across all conditions:

\begin{table}[ht]
\centering
\caption{Configuration for the Terminal-Bench-2 workflow-step experiment.}
\label{tab:terminal_bench_config}
\begin{tabular}{ll}
\toprule
\textbf{Component} & \textbf{Setting} \\
\midrule
Planner & \texttt{openrouter/qwen/qwen3.5-plus-02-15} \\
Solver & \texttt{openrouter/z-ai/glm-5} \\
Execution backend & Harbor Docker (local) \\
Episode budget & 50 episodes per run \\
Attempts per cell & \texttt{N\_ATTEMPTS=1} per $(\text{task},k)$\\
Verifier & Official Terminal-Bench verifier (per task) \\
Prompt scaffold & Fixed; only the workflow block varies with $k$ \\
\bottomrule
\end{tabular}
\end{table}

All experiments are run locally on a MacBook with an Apple M4 chip.

\subsubsection{Pipeline}

Each $(\text{task},k)$ cell follows a two-stage pipeline.

\textbf{Stage 1: Workflow Generation.}
For each task, the planner generates one workflow per granularity $k\!\in\!\{1,\dots,10\}$, producing exactly $k$ ordered steps. The planner receives the raw Terminal-Bench task instruction together with a fixed generation prompt that requests $k$ ordered intermediate actions or verification points and explicitly prohibits fabricated command outputs or verifier results unsupported by the instruction. Each task's workflows are saved as a single JSON object indexed by $k$.

\textbf{Stage 2: Solver Evaluation.}
For each $(\text{task},k)$, the solver \footnote{\url{https://github.com/stanford-iris-lab/meta-harness-tbench2-artifact}} receives a prompt of the form
$$
\Pi_k(x) \;=\; \texttt{Template}\bigl(x,\ \texttt{Format}(W_k(x))\bigr),
$$
where $x$ is the task instruction, $W_k(x)$ is the $k$-step workflow, and \texttt{Template} is a fixed scaffold containing three blocks: \texttt{\#\#~task} (verbatim $x$), \texttt{\#\#~guidance} (a four-bullet $k$-invariant rubric on inspecting the verifier, reproducing failures, and minimal local edits), and \texttt{\#\#~workflow} (the numbered $k$-step list). The solver is then run once under the fixed Harbor Docker environment with \texttt{MAX\_EPISODES=50}, and success is determined by the verifier's reward. 

\subsubsection{Metric}

For task $c$, step count $k$, and trial $r$, define the per-cell outcome as
$$
s_{c,k,r}\;=\;\mathbf{1}\!\left\{\texttt{verifier\_result.rewards.reward}>0\right\}.
$$
The aggregate pass rate at granularity $k$ is
$$
\mathrm{PassRate}(k)
\;=\;
\frac{n_{\mathrm{succ}}(k)}{n_{\mathrm{valid}}(k)},
$$
where $n_{\mathrm{valid}}(k)$ is the number of evaluated cells at $k$ (excluding missing or errored runs) and $n_{\mathrm{succ}}(k)$ is the number of cells with $s_{c,k,r}=1$. For successful cells, we additionally report \texttt{episodes\_to\_success}, taken from the agent's \texttt{n\_episodes} metadata.

\subsection{Plot-Reasoning Hallucination Setup}
\label{app:plot_reasoning_hallucination_setup}

\begin{figure}[t]
    \centering
    \begin{subfigure}{0.32\textwidth}
        \centering
        \includegraphics[width=\linewidth]{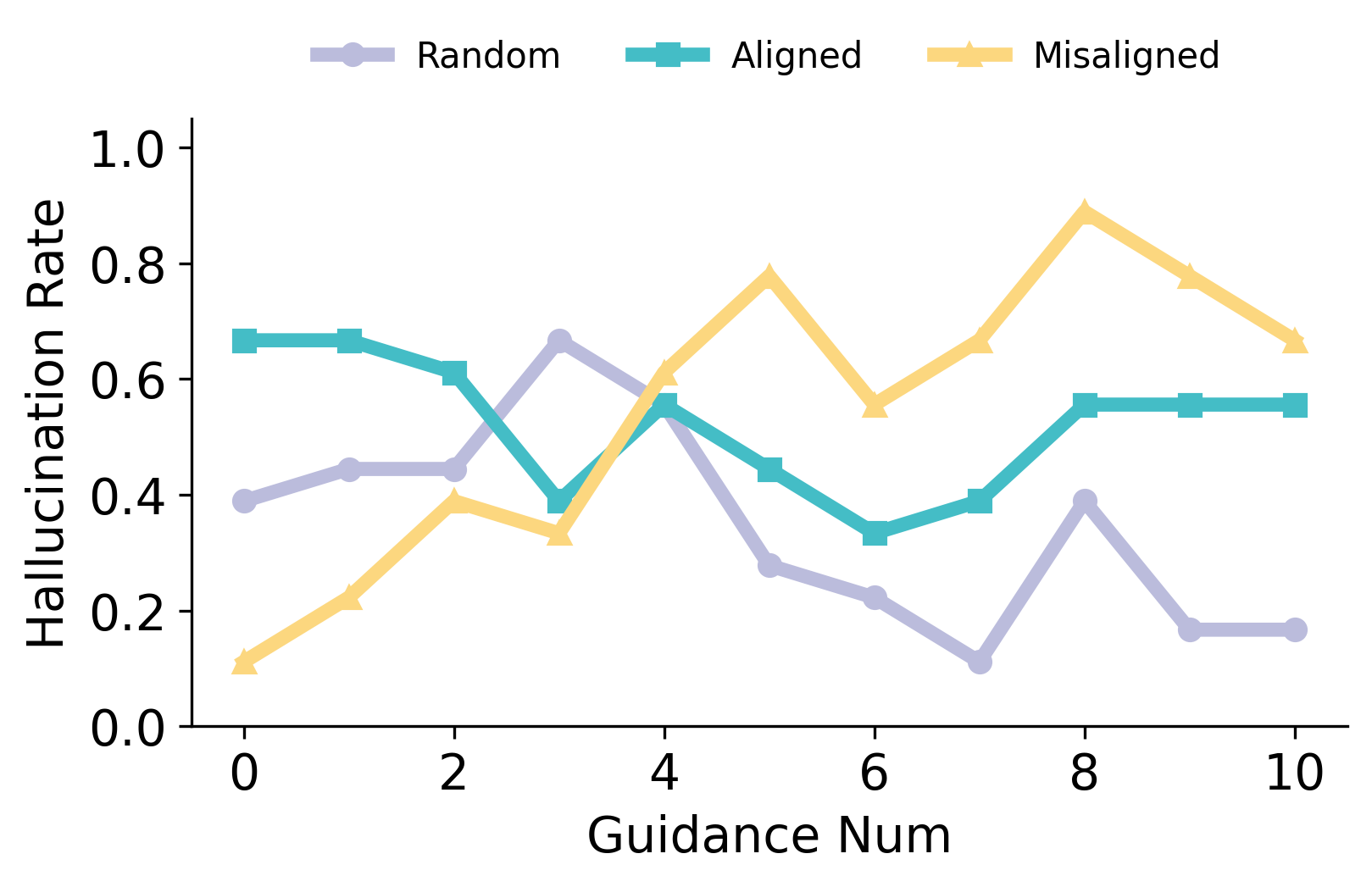}
        \caption{Qwen3.5-plus-02-15}
        \label{fig:hallucination_qwen}
    \end{subfigure}
    \hfill
    \begin{subfigure}{0.32\textwidth}
        \centering
        \includegraphics[width=\linewidth]{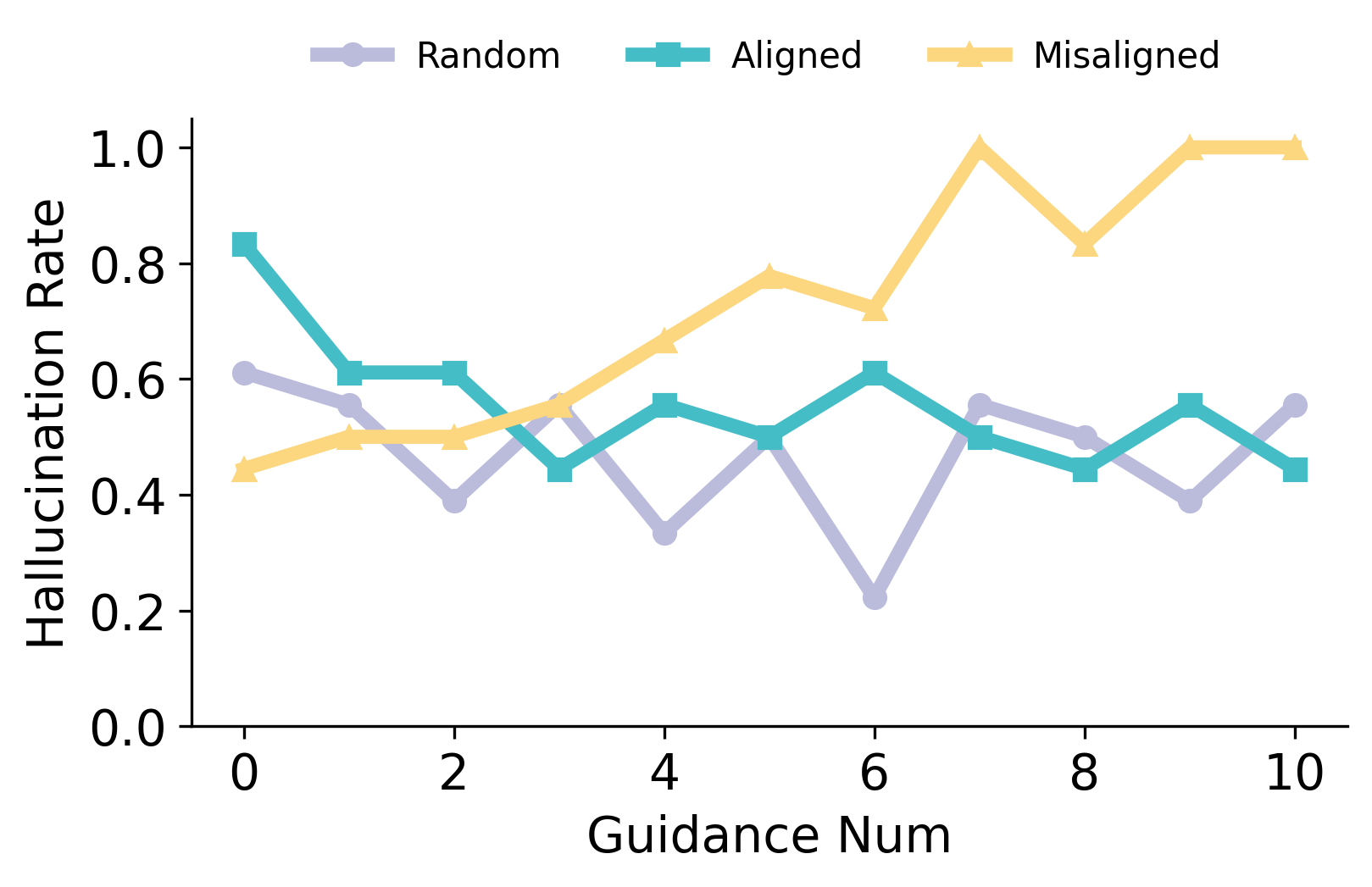}
        \caption{GPT-OSS-120B}
        \label{fig:hallucination_gpt}
    \end{subfigure}
    \hfill
    \begin{subfigure}{0.32\textwidth}
        \centering
        \includegraphics[width=\linewidth]{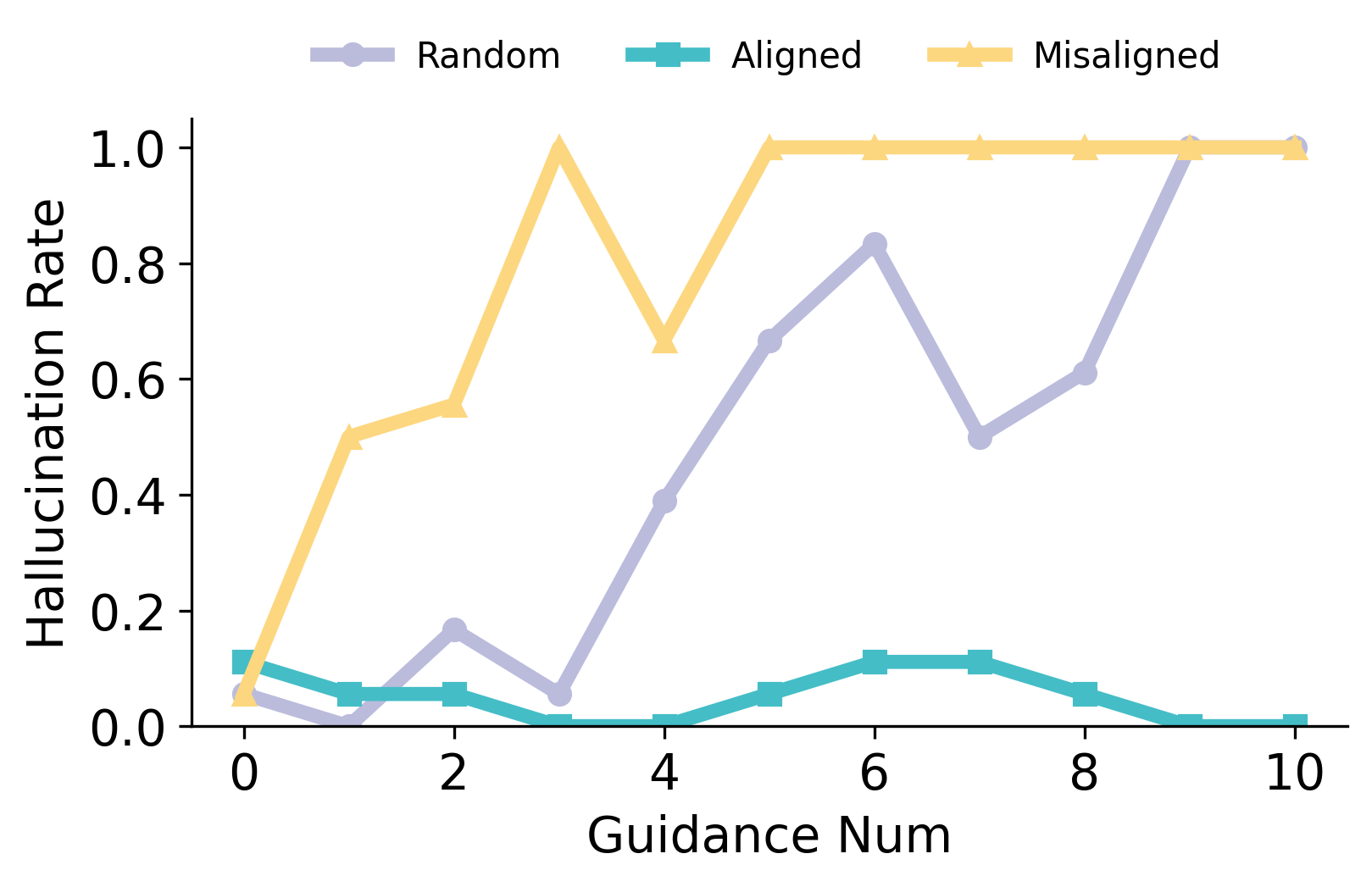}
        \caption{Gemini-3-flash-preview}
        \label{fig:hallucination_gemini}
    \end{subfigure}
    \caption{
    \textbf{Granularity--capability alignment in a controlled addition task.} (a) Pass rate varies with the number of subgoals and depends on the match between harness granularity and agent capability.
    (b) The mean absolute final bias increases as the number of subgoals grows, showing that overly fine task decomposition can accumulate larger terminal error.
    }
    \label{fig:granularity_two_panel}
\end{figure}

We further evaluate harness guidance in a visibility-limited chart reasoning task. This experiment complements the controlled simulations by testing whether additional guidance reduces or amplifies unsupported claims in a more realistic analysis setting. The central question is: when the model only observes a truncated Plotly representation rather than the rendered figure, how does the amount and type of guidance affect hallucination?

\subsubsection{Task Construction}

Each task is a chart-analysis problem. Instead of showing the model the rendered image, we provide only the raw Plotly trace representation, denoted as \texttt{plotly\_repr}. This input contains partial structural information about the chart, but it does not provide full visual access to the figure. Therefore, the model must reason under limited observability.

We generate six chart types:
\[
\texttt{box},\quad
\texttt{line},\quad
\texttt{scatter},\quad
\texttt{histogram},\quad
\texttt{violin},\quad
\texttt{heatmap}.
\]
For each chart type, we generate a large synthetic chart and store its Plotly representation as a task instance. The random seed is fixed to ensure reproducibility. Each task contains a chart name and a string-valued \texttt{plotly\_repr}, which is later inserted into the analysis prompt.

\subsubsection{Guidance Conditions}

The harness augments the chart-analysis prompt with a set of textual guidance rules. We vary the number of injected rules, denoted by $k$. When $k=0$, the model receives no additional guidance beyond the base prompt. As $k$ increases, the harness imposes more requirements on how the model should analyze the chart.

We compare three guidance sources. The first source is the original guidance set, which contains general analysis requirements. The second source is an anti-hallucination guidance set, which explicitly instructs the model to respect the limited visibility of the Plotly representation and avoid unsupported visual claims. The third source is a mixed condition, where guidance rules are sampled from the two sources with equal probability. Thus, the experiment separates the effect of guidance quantity from the effect of guidance quality.

\subsubsection{Generation and Judging Pipeline}

Each experimental case follows the same two-stage pipeline. First, the analysis model receives the base chart-analysis prompt, the task-specific \texttt{plotly\_repr}, and $k$ sampled guidance rules. It then produces a natural-language analysis of the chart. Second, the judge model reads both the input context and the generated analysis, and assigns one of three labels:
\[
\texttt{GROUNDED},\quad
\texttt{HALLUCINATION},\quad
\texttt{ERROR}.
\]

A response is labeled \texttt{GROUNDED} if the analysis stays within what can be reasonably inferred from the provided representation. It is labeled \texttt{HALLUCINATION} if it makes unsupported claims, overstates visual details, or fails to acknowledge the limited observability of the chart. Cases with API failures or invalid outputs are labeled \texttt{ERROR} and excluded from the valid denominator.

\subsubsection{Metrics}

For each guidance count $k$, we aggregate results across chart types and trials. Let $n_{\mathrm{valid}}(k)$ be the number of non-error cases, and let $n_{\mathrm{hall}}(k)$ be the number of cases labeled as hallucination. The main metric is the hallucination rate:
\[
\mathrm{HallucinationRate}(k)
=
\frac{n_{\mathrm{hall}}(k)}{n_{\mathrm{valid}}(k)}.
\]
We also report the number of error cases to ensure that a lower hallucination rate is not caused by a large number of invalid model outputs.

\subsubsection{Experimental Variables}

The independent variables are the guidance count $k$ and the guidance source. By default, we evaluate
$$
k \in \{0,1,\ldots,10\},
$$
with three trials for each $k$ and each chart type. For a fair comparison, the analysis model, judge model, task set, random seed, and tested values of $k$ are kept fixed across guidance strategies.

\subsubsection{Interpretation}

This experiment is designed to test whether stronger guidance is always better. If increasing $k$ consistently lowers hallucination rate, then additional guidance acts as useful inference-time constraint. If hallucination rate increases with $k$, then excessive guidance may encourage the model to satisfy more textual requirements than the limited evidence can support. Comparing the original, anti-hallucination, and mixed guidance conditions further distinguishes whether the outcome is driven by the number of guidance rules or by their alignment with the available evidence.

%% file: text/999_1_proof.tex
\section{Omitted Proofs}

\input{text/999_1_0_loss_decouple}
\input{text/999_1_1_step_alignment}

\input{text/999_1_2_guidance}

\input{text/999_1_3_partial_harnessing}

%% file: text/999_1_0_loss_decouple.tex
\subsection{Recoverability Construction and Process Loss}
\label{app:recoverability_factorization}

This appendix formalizes the chain-rule factorization $P_h(\mathsf{Succ}_x(\tau_h) \mid x) = \prod_t \bar p_t(h;x)$ stated without proof in Section~\ref{sec:preliminary}, on which all subsequent stagewise bounds rely. The factorization requires a precise notion of when an intermediate prefix can still reach a successful continuation, so Appendix~\ref{app:recoverable-prefixes-goal-consistency} first defines the recoverable prefix sets $\mathcal{R}_t(x,h)$, identifies the stagewise event $B_t$ with membership in $\mathcal{R}_t$, and isolates the goal-consistency condition that aligns these intermediate notions with final success. Under this condition, Appendix~\ref{app:chain-rule-identity} then derives the factorization by chain rule and converts it into the additive log form invoked by the bounds in Appendix~\ref{sec:granularity_alignment} and Appendix~\ref{app:guidance_action_space_filtering}.

\subsubsection{Recoverable Prefixes and Goal Consistency}
\label{app:recoverable-prefixes-goal-consistency}

Recall from Section~\ref{sec:preliminary} the completed prefix
$$
K_t := (x, g_1, \tau_1, \ldots, g_t, \tau_t),
\qquad
K_0 := x,
$$
which records the full execution trace through stage $t$. For each intermediate stage $t < T$, define the recoverable set
$$
\mathcal{R}_t(x, h)
:=
\left\{
K_t :
\exists\, (\tau_{t+1}, \ldots, \tau_T)
\text{ such that }
\mathsf{Succ}_x(\tau_h)
\right\},
$$
the set of completed prefixes from which some continuation under the remaining plan $(g_{t+1}, \ldots, g_T)$ reaches $y^\star(x)$. At the terminal stage we set $\mathcal{R}_T(x, h) := \{K_T : \mathsf{Succ}_x(\tau_h)\}$, so a terminal prefix is recoverable iff it is successful. The stagewise recoverability event of Section~\ref{sec:preliminary} is then
$$
B_t := \{K_t \in \mathcal{R}_t(x, h)\},
\qquad
t = 1, \ldots, T,
$$
and the cumulative events $E_t := \bigcap_{s \le t} B_s$ form a nested sequence $E_T \subseteq \cdots \subseteq E_1$.

We call the harness plan \emph{goal-consistent} if
$$
E_T \equiv \mathsf{Succ}_x(\tau_h),
$$
that is, if remaining recoverable through every stage is equivalent to final success. This rules out decompositions whose intermediate recoverability notions are misaligned with the final answer, and it is the regime under which the chain-rule identity below holds.

\subsubsection{Chain-Rule Identity}
\label{app:chain-rule-identity}

Under goal consistency, $P_h(\mathsf{Succ}_x(\tau_h) \mid x) = P_h(E_T \mid x)$. Applying the chain rule to the nested events $E_1, \ldots, E_T$ and using $E_t = E_{t-1} \cap B_t$ together with $E_{t-1} = B_{<t}$,
$$
P_h(\mathsf{Succ}_x(\tau_h) \mid x)
=
\prod_{t=1}^{T}
P_h(B_t \mid B_{<t}, x)
=
\prod_{t=1}^{T}
\bar p_t(h; x),
$$
which is the factorization stated in Section~\ref{sec:preliminary}. Taking negative logarithms,
$$
-\log P_h(\mathsf{Succ}_x(\tau_h) \mid x)
=
\sum_{t=1}^{T} -\log \bar p_t(h; x),
$$
so the process loss is induced by the primitive final-success objective rather than added on top of it. Each summand penalizes a low conditional probability of preserving recoverability after one harness-level plan step, and this is the form used by the bounds in Appendix~\ref{sec:granularity_alignment} and Appendix~\ref{app:guidance_action_space_filtering}.

%% file: text/999_1_1_step_alignment.tex
\subsection{Granularity--Capability Alignment}
\label{sec:granularity_alignment}

This appendix gives the formal statement and proof for the granularity--capability alignment principle of Section~\ref{sec:granularity_capability_alignment}. The result quantifies how the harness's choice of sub-goal scale, combined with the agent's per-stage attempt budget, controls the final success probability through a per-stage mismatch penalty. Theorem~\ref{thm:granularity_capability_alignment} states the general bound under prefix-dependent stage quantities, taking infimum over all recoverable prefixes; Remark~\ref{rem:deterministic_stage} then specializes this to the deterministic regime that recovers the simplified form stated in the main text. Building on this specialization, Corollary~\ref{cor:coarse_fine_granularity} characterizes when a uniform $T$-step decomposition incurs no mismatch penalty: the sub-goal size $L_x/T$ must lie in the union of cumulative controllability windows reachable within the attempt budget, with deviations on either side producing the coarse-grained and fine-grained failure regimes observed in the main text.

\begin{theorem}[Granularity--capability alignment bound with finite attempts]
\label{thm:granularity_capability_alignment}
Fix a task $x$ and a harness $h=(\kappa,\lambda,\psi)$. Let
$$
\Delta_h(x)=(g_1,\dots,g_T),
\qquad
T=T_h(x),
$$
be the harness-induced sub-goal sequence, and let
$
\tau_h(x)=(g_1,\tau_1,\dots,g_T,\tau_T)
$
be the harness-conditioned execution trajectory. Recall that the completed
prefix after the first $t$ harness-level sub-goals is
$$
K_t=(x,g_1,\tau_1,\dots,g_t,\tau_t),
\qquad
K_0:=x.
$$

For each stage $t$ and each completed prefix $K_{t-1}$ satisfying $B_{<t}$,
suppose there exists a latent progress coordinate $\phi_x$. Define the intended
latent progress required by the harness-specified sub-goal $g_t$ as
$$
\ell_t(K_{t-1})
:=
\phi_x(g_t)-\phi_x(K_{t-1}).
$$

At stage $t$, suppose the agent may make at most $M_t$ primitive attempts before
the stage fails. For $m=1,\dots,M_t$, let
$
Z_{t,m}
$
denote the cumulative latent progress made after $m$ primitive attempts within
stage $t$, measured from the prefix $K_{t-1}$.

Assume the following conditions hold for every stage $t$ and almost surely over
completed prefixes $K_{t-1}$ satisfying $B_{<t}$.

\textbf{(i) Finite-attempt recoverability tube.}
There exists a tolerance $\varepsilon_t(K_{t-1})\ge 0$ such that preserving
recoverability after executing sub-goal $g_t$ requires landing near the intended
milestone within the finite attempt budget:
$$
B_t
\subseteq
\bigcup_{m=1}^{M_t}
\left\{
\left|
Z_{t,m}
-
\ell_t(K_{t-1})
\right|
\le
\varepsilon_t(K_{t-1})
\right\}.
$$

\textbf{(ii) Cumulative controllability windows.}
For each $m=1,\dots,M_t$, there exist quantities
$$
\mu_{t,m}^-(K_{t-1})
\le
\mu_{t,m}^+(K_{t-1}),
\qquad
\sigma_{t,m}(K_{t-1})>0,
$$
and a mean
$$
\mu_{t,m}(K_{t-1})
\in
[
\mu_{t,m}^-(K_{t-1}),
\mu_{t,m}^+(K_{t-1})
],
$$
such that, for all $u\ge 0$,
$$
P_h\left(
Z_{t,m}-\mu_{t,m}(K_{t-1})\ge u
\mid K_{t-1}
\right)
\le
\exp\left(
-\frac{u^2}{2\sigma_{t,m}^2(K_{t-1})}
\right),
$$
and
$$
P_h\left(
\mu_{t,m}(K_{t-1})-Z_{t,m}\ge u
\mid K_{t-1}
\right)
\le
\exp\left(
-\frac{u^2}{2\sigma_{t,m}^2(K_{t-1})}
\right).
$$

\textbf{(iii) Boundary loss.}
Crossing the $t$-th harness-specified sub-goal boundary incurs an irreducible
coordination loss $\eta_t(K_{t-1})\ge 0$, so that
$$
P_h(B_t\mid K_{t-1})
\le
\exp\bigl(-\eta_t(K_{t-1})\bigr)
P_h\left(
\bigcup_{m=1}^{M_t}
\left\{
\left|
Z_{t,m}
-
\ell_t(K_{t-1})
\right|
\le
\varepsilon_t(K_{t-1})
\right\}
\mid K_{t-1}
\right).
$$

For each $m=1,\dots,M_t$, define the prefix-wise finite-attempt mismatch
$$
a_{t,m}(K_{t-1})
:=
\frac{
\left(
d\left(
\ell_t(K_{t-1}),
[
\mu_{t,m}^-(K_{t-1}),
\mu_{t,m}^+(K_{t-1})
]
\right)
-
\varepsilon_t(K_{t-1})
\right)_+^2
}{
2\sigma_{t,m}^2(K_{t-1})
}.
$$
Define the best finite-attempt mismatch at stage $t$ as
$$
\rho_t^{(M)}(K_{t-1})
:=
\min_{1\le m\le M_t}
a_{t,m}(K_{t-1}).
$$

Let
$$
\mathcal R_{t-1}(h,x)
:=
\operatorname{supp}_{P_h(\cdot\mid B_{<t},x)}(K_{t-1})
$$
denote the set of recoverable prefixes that can arise before stage $t$. Define
the stage-wise certified finite-attempt granularity loss as
$$
\gamma_t^{(M)}(h;x)
:=
\inf_{K_{t-1}\in\mathcal R_{t-1}(h,x)}
\left\{
\eta_t(K_{t-1})
+
\left(
\rho_t^{(M)}(K_{t-1})-\log M_t
\right)_+
\right\}.
$$

Then the final success probability satisfies
$$
-\log P_h(\mathrm{Succ}_x(\tau_h)\mid x)
\ge
\sum_{t=1}^{T_h(x)}
\gamma_t^{(M)}(h;x).
$$
Equivalently,
$$
P_h(\mathrm{Succ}_x(\tau_h)\mid x)
\le
\exp\left(
-
\sum_{t=1}^{T_h(x)}
\gamma_t^{(M)}(h;x)
\right).
$$
\end{theorem}

\begin{remark}[Deterministic stage quantities]
\label{rem:deterministic_stage}
When the prefix-dependent quantities
$$
\ell_t(K_{t-1}),\;
\mu_{t,m}^\pm(K_{t-1}),\;
\sigma_{t,m}(K_{t-1}),\;
\varepsilon_t(K_{t-1}),\;
\eta_t(K_{t-1})
$$
are constant on $B_{<t}$---a regime we denote without the $K_{t-1}$ argument---the infimum in $\gamma_t^{(M)}(h;x)$ collapses to its argument, and Theorem~\ref{thm:granularity_capability_alignment} reduces to
$$
P_h(\mathrm{Succ}_x(\tau_h)\mid x)
\le
\exp\left(
-\sum_{t=1}^{T_h(x)}
\bigl[\eta_t + (\rho_t^{(M)}-\log M_t)_+\bigr]
\right),
$$
which is the form stated in the main text. The two corollaries below operate in this regime.
\end{remark}

\begin{proof}[Proof of Theorem~\ref{thm:granularity_capability_alignment}]
Fix a task $x$ and a harness $h$, and write $T=T_h(x)$. All probabilities are
taken under the harness-conditioned execution distribution induced by $q_h$.

Fix a stage $t$ and a completed prefix $K_{t-1}$ satisfying $B_{<t}$. For each
$m=1,\dots,M_t$, define the finite-attempt tube event
$$
E_{t,m}(K_{t-1})
:=
\left\{
\left|
Z_{t,m}
-
\ell_t(K_{t-1})
\right|
\le
\varepsilon_t(K_{t-1})
\right\}.
$$
The event that the agent hits the recoverability tube within the finite attempt
budget is
$$
E_t^{(M)}(K_{t-1})
:=
\bigcup_{m=1}^{M_t}
E_{t,m}(K_{t-1}).
$$

By the boundary-loss assumption,
$$
P_h(B_t\mid K_{t-1})
\le
\exp\bigl(-\eta_t(K_{t-1})\bigr)
P_h(E_t^{(M)}(K_{t-1})\mid K_{t-1}).
$$
It remains to upper bound the finite-attempt tube probability
$P_h(E_t^{(M)}(K_{t-1})\mid K_{t-1})$.

Fix an attempt count $m\in\{1,\dots,M_t\}$. Let
$$
I_{t,m}(K_{t-1})
:=
[
\mu_{t,m}^-(K_{t-1}),
\mu_{t,m}^+(K_{t-1})
],
$$
and define
$$
d_{t,m}(K_{t-1})
:=
d\left(
\ell_t(K_{t-1}),
I_{t,m}(K_{t-1})
\right).
$$
We first show that
$$
P_h(E_{t,m}(K_{t-1})\mid K_{t-1})
\le
\exp\left(
-
a_{t,m}(K_{t-1})
\right).
$$

There are three cases.

First suppose
$$
\ell_t(K_{t-1})
>
\mu_{t,m}^+(K_{t-1}).
$$
Then
$$
d_{t,m}(K_{t-1})
=
\ell_t(K_{t-1})
-
\mu_{t,m}^+(K_{t-1}).
$$
Since
$$
\mu_{t,m}(K_{t-1})
\le
\mu_{t,m}^+(K_{t-1}),
$$
on the event $E_{t,m}(K_{t-1})$ we have
$$
Z_{t,m}
\ge
\ell_t(K_{t-1})
-
\varepsilon_t(K_{t-1}).
$$
Therefore,
$$
Z_{t,m}-\mu_{t,m}(K_{t-1})
\ge
\ell_t(K_{t-1})
-
\varepsilon_t(K_{t-1})
-
\mu_{t,m}^+(K_{t-1})
=
d_{t,m}(K_{t-1})
-
\varepsilon_t(K_{t-1}).
$$
If
$$
d_{t,m}(K_{t-1})
>
\varepsilon_t(K_{t-1}),
$$
the upper-tail sub-Gaussian bound gives
$$
P_h(E_{t,m}(K_{t-1})\mid K_{t-1})
\le
\exp\left(
-
\frac{
\left(
d_{t,m}(K_{t-1})
-
\varepsilon_t(K_{t-1})
\right)^2
}{
2\sigma_{t,m}^2(K_{t-1})
}
\right).
$$
If
$$
d_{t,m}(K_{t-1})
\le
\varepsilon_t(K_{t-1}),
$$
the desired bound is trivial because the right-hand side equals $1$.

Second suppose
$$
\ell_t(K_{t-1})
<
\mu_{t,m}^-(K_{t-1}).
$$
Then
$$
d_{t,m}(K_{t-1})
=
\mu_{t,m}^-(K_{t-1})
-
\ell_t(K_{t-1}).
$$
Since
$$
\mu_{t,m}(K_{t-1})
\ge
\mu_{t,m}^-(K_{t-1}),
$$
on the event $E_{t,m}(K_{t-1})$ we have
$$
Z_{t,m}
\le
\ell_t(K_{t-1})
+
\varepsilon_t(K_{t-1}).
$$
Therefore,
$$
\mu_{t,m}(K_{t-1})-Z_{t,m}
\ge
\mu_{t,m}^-(K_{t-1})
-
\ell_t(K_{t-1})
-
\varepsilon_t(K_{t-1})
=
d_{t,m}(K_{t-1})
-
\varepsilon_t(K_{t-1}).
$$
The lower-tail sub-Gaussian bound gives
$$
P_h(E_{t,m}(K_{t-1})\mid K_{t-1})
\le
\exp\left(
-
\frac{
\left(
d_{t,m}(K_{t-1})
-
\varepsilon_t(K_{t-1})
\right)_+^2
}{
2\sigma_{t,m}^2(K_{t-1})
}
\right).
$$

Third suppose
$$
\ell_t(K_{t-1})
\in
I_{t,m}(K_{t-1}).
$$
Then $d_{t,m}(K_{t-1})=0$, and the desired bound reduces to
$$
P_h(E_{t,m}(K_{t-1})\mid K_{t-1})\le 1,
$$
which is immediate.

Combining the three cases, for every $m=1,\dots,M_t$,
$$
P_h(E_{t,m}(K_{t-1})\mid K_{t-1})
\le
\exp\left(
-
a_{t,m}(K_{t-1})
\right).
$$

Now apply the union bound over the finite attempt budget:
$$
P_h(E_t^{(M)}(K_{t-1})\mid K_{t-1})
=
P_h\left(
\bigcup_{m=1}^{M_t}
E_{t,m}(K_{t-1})
\mid K_{t-1}
\right)
\le
\sum_{m=1}^{M_t}
P_h(E_{t,m}(K_{t-1})\mid K_{t-1}).
$$
Using the previous bound,
$$
P_h(E_t^{(M)}(K_{t-1})\mid K_{t-1})
\le
\sum_{m=1}^{M_t}
\exp\left(
-
a_{t,m}(K_{t-1})
\right).
$$
Since
$$
\rho_t^{(M)}(K_{t-1})
=
\min_{1\le m\le M_t}
a_{t,m}(K_{t-1}),
$$
we have
$$
\sum_{m=1}^{M_t}
\exp\left(
-
a_{t,m}(K_{t-1})
\right)
\le
M_t
\exp\left(
-
\rho_t^{(M)}(K_{t-1})
\right).
$$
Therefore,
$$
P_h(E_t^{(M)}(K_{t-1})\mid K_{t-1})
\le
\min\left\{
1,\,
M_t
\exp\left(
-
\rho_t^{(M)}(K_{t-1})
\right)
\right\}.
$$
Equivalently,
$$
P_h(E_t^{(M)}(K_{t-1})\mid K_{t-1})
\le
\exp\left(
-
\left(
\rho_t^{(M)}(K_{t-1})
-
\log M_t
\right)_+
\right).
$$

Combining this with the boundary-loss assumption gives
$$
P_h(B_t\mid K_{t-1})
\le
\exp\left(
-
\eta_t(K_{t-1})
-
\left(
\rho_t^{(M)}(K_{t-1})
-
\log M_t
\right)_+
\right).
$$

By the definition of $\gamma_t^{(M)}(h;x)$, for every
$K_{t-1}\in\mathcal R_{t-1}(h,x)$,
$$
\eta_t(K_{t-1})
+
\left(
\rho_t^{(M)}(K_{t-1})
-
\log M_t
\right)_+
\ge
\gamma_t^{(M)}(h;x).
$$
Hence, almost surely on $B_{<t}$,
$$
P_h(B_t\mid K_{t-1})
\le
\exp\bigl(-\gamma_t^{(M)}(h;x)\bigr).
$$

Taking expectation over $K_{t-1}$ conditional on $B_{<t}$ gives
$$
\bar p_t(h;x)
=
P_h(B_t\mid B_{<t},x)
\le
\exp\bigl(-\gamma_t^{(M)}(h;x)\bigr).
$$

By the recoverability factorization in
Section~\ref{sec:preliminary},
$$
P_h(\mathrm{Succ}_x(\tau_h)\mid x)
=
\prod_{t=1}^{T_h(x)}
\bar p_t(h;x).
$$
Therefore,
$$
P_h(\mathrm{Succ}_x(\tau_h)\mid x)
\le
\prod_{t=1}^{T_h(x)}
\exp\bigl(-\gamma_t^{(M)}(h;x)\bigr)
=
\exp\left(
-
\sum_{t=1}^{T_h(x)}
\gamma_t^{(M)}(h;x)
\right).
$$
Taking negative logarithms proves the claim.
\end{proof}

\begin{corollary}[Failure of overly coarse and overly fine decompositions]
\label{cor:coarse_fine_granularity}
Adopt the deterministic regime of Remark~\ref{rem:deterministic_stage}, and
assume in addition that the harness uses a uniform $T$-step decomposition with
$\ell_t=L_x/T$ for all $t$, that each primitive attempt has a one-step
controllable window $[r^-,r^+]$ and one-step variation $\sigma$ so that
$$
[\mu_{t,m}^-,\mu_{t,m}^+]=[mr^-,mr^+],
\qquad
\sigma_{t,m}^2=m\sigma^2,
$$
and that $\eta_t=\eta$, $\varepsilon_t=\varepsilon$, $M_t=M$. Then
$$
-\log P_h(\mathrm{Succ}_x(\tau_h)\mid x)
\ge
T
\left[
\eta
+
\left(
\min_{1\le m\le M}
\frac{
(d(L_x/T,[mr^-,mr^+])-\varepsilon)_+^2
}{
2m\sigma^2
}
-
\log M
\right)_+
\right],
$$
and the certified mismatch term vanishes if and only if
$$
\frac{L_x}{T}
\in
\bigcup_{m=1}^{M}
[mr^- - \varepsilon,\,mr^+ + \varepsilon].
$$
Avoiding both coarse-grained ($L_x/T$ above the union) and fine-grained ($L_x/T$ below the union) failure regimes therefore requires the uniform sub-goal size to be reachable by some cumulative progress scale within the finite attempt budget.
\end{corollary}

\begin{proof}[Proof of Corollary~\ref{cor:coarse_fine_granularity}]
Substituting the uniform and additive assumptions into
$\rho_t^{(M)}$ from Remark~\ref{rem:deterministic_stage} gives
$$
\rho_t^{(M)}
=
\min_{1\le m\le M}
\frac{
(d(L_x/T,[mr^-,mr^+])-\varepsilon)_+^2
}{
2m\sigma^2
},
$$
which is independent of $t$. Plugging into the bound in
Remark~\ref{rem:deterministic_stage} and using $\eta_t=\eta$, $M_t=M$ for all
$t$ yields the stated lower bound on $-\log P_h$.

The mismatch term $\rho_t^{(M)}$ vanishes precisely when
$d(L_x/T,[mr^-,mr^+])\le\varepsilon$ for some $m\in\{1,\dots,M\}$, which is
equivalent to $L_x/T\in[mr^- - \varepsilon, mr^+ + \varepsilon]$ for that $m$.
The union over $m$ characterizes the set of sub-goal sizes for which no
mismatch penalty is paid; outside this union, $L_x/T$ either exceeds every
$mr^+ + \varepsilon$ (coarse-grained failure) or falls below every
$mr^- - \varepsilon$ except at small $m$, where the small reachable scales
cannot match $L_x/T$ either (fine-grained failure).
\end{proof}

%% file: text/999_1_2_guidance.tex
\subsection{Guidance as Action-Space Filtering}
\label{app:guidance_action_space_filtering}

This appendix gives the formal statement and proof for the guidance--evidence alignment principle of Section~\ref{sec:guidance_action_space_filtering}. The result casts guidance as a reweighting of the base stage-level trajectory distribution and shows that its effect on stage recoverability is governed by a single one-dimensional quantity: the log-ratio of average retention weights on recoverable versus non-recoverable trajectories. Appendix~\ref{app:stage-level-filtering-model} sets up the filtering model, defining the guidance-filtered distribution $Q_{\lambda_t}$ and the conditional retention weights $\bar W_{t,R_t^{\mathrm{stg}}}, \bar W_{t,R^c}$ that the rest of the analysis is built on. Appendix~\ref{app:action-space-filtering-identity} then uses these objects to prove the central identity, expressing the filtered recoverability probability as a sigmoid in the base log-odds plus the retention gap $\Gamma_{t,\lambda_t}(K_{t-1})$. Building on this identity, Appendix~\ref{app:guidance-action-space-alignment-theorem} states the alignment theorem characterizing when guidance helps, hurts, or leaves stage recoverability unchanged, together with the process-level integration that lifts per-stage retention gaps into a final-success bound. Appendix~\ref{app:hallucination_special_case} closes by interpreting pseudo-compliant hallucination as a concrete instance of a negative retention gap.

\subsubsection{Stage-Level Filtering Model}
\label{app:stage-level-filtering-model}

Fix a task $x$, a harness stage $t$, and a completed prefix $K_{t-1}$ satisfying $B_{<t}$. At this stage, the harness specifies a sub-goal $g_t$. Let
$
Q_0(\tau_t\mid K_{t-1},g_t)
$
be the base stage-level trajectory distribution.

Guidance is represented by a nonnegative measurable retention weight $W_{t,\lambda_t}(K_{t-1},g_t,\tau_t)\ge 0$, where $\lambda_t$ denotes the amount of guidance imposed at stage $t$. The guidance-filtered distribution is
$$
Q_{\lambda_t}(\tau_t\mid K_{t-1},g_t)
=
\frac{
Q_0(\tau_t\mid K_{t-1},g_t)\,
W_{t,\lambda_t}(K_{t-1},g_t,\tau_t)
}{
Z_{t,\lambda_t}(K_{t-1},g_t)
},
\qquad
Z_{t,\lambda_t}(K_{t-1},g_t)
:=
\mathbb E_{Q_0}[W_{t,\lambda_t}],
$$
and we assume throughout that $0<Z_{t,\lambda_t}<\infty$.

This formulation is intentionally general: it captures inference-time guidance mechanisms that retain, downweight, or upweight candidate stage trajectories. Hard action-space pruning is the special case where $W_{t,\lambda_t}$ is an indicator over remaining trajectories.

Let
$$
\mathcal R_t^{\mathrm{stg}}(K_{t-1})
:=
\left\{
\tau_t:
K_t=(K_{t-1},g_t,\tau_t)
\text{ remains recoverable}
\right\}
$$
denote the recoverable cross-section of the stage-level trajectory space, and assume the base recoverability probability is non-degenerate, $0<Q_0(\mathcal R_t^{\mathrm{stg}}\mid K_{t-1},g_t)<1$. Define the conditional average retention weights on the recoverable set and its complement,
$$
\bar W_{t,R_t^{\mathrm{stg}}}(K_{t-1})
:=
\mathbb E_{Q_0(\cdot\mid \mathcal R_t^{\mathrm{stg}},K_{t-1},g_t)}[W_{t,\lambda_t}],
\qquad
\bar W_{t,R^c}(K_{t-1})
:=
\mathbb E_{Q_0(\cdot\mid \mathcal R_t^c,K_{t-1},g_t)}[W_{t,\lambda_t}],
$$
and assume $0<\bar W_{t,R_t^{\mathrm{stg}}}, \bar W_{t,R^c}<\infty$. Boundary cases where one of these quantities is zero are handled by the same identity in the extended-real sense, with $\sigma(+\infty)=1$ and $\sigma(-\infty)=0$.

\subsubsection{Action-Space Filtering Identity}
\label{app:action-space-filtering-identity}

\begin{lemma}[Action-space filtering identity]
\label{lem:action_space_filtering_identity}
Under the conditions of the previous subsection, define the base recoverability log-odds
$$
\omega_{t}^0(K_{t-1})
:=
\log
\frac{
Q_0(\mathcal R_t^{\mathrm{stg}}(K_{t-1})\mid K_{t-1},g_t)
}{
Q_0(\mathcal R_t^{\mathrm{stg}}(K_{t-1})^c\mid K_{t-1},g_t)
}
$$
and the guidance retention gap
$$
\Gamma_{t,\lambda_t}(K_{t-1})
:=
\log\bar W_{t,R_t^{\mathrm{stg}}}(K_{t-1})
-
\log\bar W_{t,R^c}(K_{t-1}).
$$
Then
$$
Q_{\lambda_t}(\mathcal R_t^{\mathrm{stg}}(K_{t-1})\mid K_{t-1},g_t)
=
\sigma\bigl(
\omega_{t}^0(K_{t-1})
+
\Gamma_{t,\lambda_t}(K_{t-1})
\bigr),
$$
where $\sigma(u)=1/(1+e^{-u})$.
\end{lemma}

\begin{proof}
Suppress the dependence on $(K_{t-1},g_t)$ and write $\mathcal R_t=\mathcal R_t^{\mathrm{stg}}(K_{t-1})$. By definition of the filtered distribution,
$$
\frac{Q_{\lambda_t}(\mathcal R_t)}{Q_{\lambda_t}(\mathcal R_t^c)}
=
\frac{Q_0(\mathcal R_t)\,\mathbb E_{Q_0(\cdot\mid \mathcal R_t)}[W_{t,\lambda_t}]}{Q_0(\mathcal R_t^c)\,\mathbb E_{Q_0(\cdot\mid \mathcal R_t^c)}[W_{t,\lambda_t}]},
$$
since the two normalizers cancel. Taking logarithms gives
$$
\log\frac{Q_{\lambda_t}(\mathcal R_t)}{Q_{\lambda_t}(\mathcal R_t^c)}
=
\omega_{t}^0
+
\Gamma_{t,\lambda_t}(K_{t-1}),
$$
and converting log-odds to probability via $\sigma$ yields the result.
\end{proof}

\subsubsection{Guidance--Action-Space Alignment}
\label{app:guidance-action-space-alignment-theorem}

\begin{theorem}[Guidance--action-space alignment]
\label{thm:guidance_action_space_alignment}
Under the conditions of Lemma~\ref{lem:action_space_filtering_identity}, guidance improves stage recoverability relative to the unguided base distribution at the same prefix if and only if $\Gamma_{t,\lambda_t}(K_{t-1})>0$, and harms it if and only if $\Gamma_{t,\lambda_t}(K_{t-1})<0$. More quantitatively, if $\Gamma_{t,\lambda_t}(K_{t-1})\ge \gamma$ for some $\gamma\ge 0$, then
$$
Q_{\lambda_t}(\mathcal R_t^{\mathrm{stg}}(K_{t-1})\mid K_{t-1},g_t)
\ge
\sigma\bigl(\omega_{t}^0(K_{t-1})+\gamma\bigr),
$$
and if $G_{t,\lambda_t}(K_{t-1})\le -\gamma$, then
$$
Q_{\lambda_t}(\mathcal R_t^{\mathrm{stg}}(K_{t-1})\mid K_{t-1},g_t)
\le
\sigma\bigl(\omega_{t}^0(K_{t-1})-\gamma\bigr).
$$
\end{theorem}

\begin{proof}
By Lemma~\ref{lem:action_space_filtering_identity}, $Q_{\lambda_t}(\mathcal R_t)=\sigma(\omega_{t}^0+\Gamma_{t,\lambda_t}(K_{t-1}))$, while the unguided probability at the same prefix is $\sigma(\omega_{t}^0)$. Both qualitative claims and quantitative bounds follow from the strict monotonicity of $\sigma$.
\end{proof}

\begin{remark}[Process-level integration]
The chain-rule identity established in Appendix~\ref{app:recoverability_factorization} integrates Theorem~\ref{thm:guidance_action_space_alignment} across stages: under goal consistency,
$$
-\log P_h(\mathsf{Succ}_x(\tau_h)\mid x)
=
\sum_{t=1}^{T_h(x)}
-\log \bar p_t(h;x),
\qquad
\bar p_t(h;x)
=
\mathbb E_h\bigl[
Q_{\lambda_t}(\mathcal R_t^{\mathrm{stg}}(K_{t-1})\mid K_{t-1},g_t)
\mid B_{<t},x
\bigr].
$$
A negative retention gap $\Gamma_{t,\lambda_t}(K_{t-1})\le -\gamma_t(K_{t-1})$ on $B_{<t}$ then yields $\bar p_t(h;x)\le \mathbb E_h[\sigma(\omega_{t}^0(K_{t-1})-\gamma_t(K_{t-1}))\mid B_{<t},x]$, so per-stage misalignment accumulates into the final-success bound.
\end{remark}

\subsubsection{Hallucination as a Special Case}
\label{app:hallucination_special_case}

The action-space filtering view also covers pseudo-compliant hallucination in content-producing tasks, where a stage trajectory $\tau_t$ may correspond to a generated answer, reasoning trace, tool-use sequence, or explanation. Let $\mathcal M_t(K_{t-1})\subseteq \mathcal R_t(K_{t-1})^c$ denote a hallucination basin: trajectories that appear compliant with the imposed guidance but introduce unsupported content and make the resulting prefix non-recoverable. If the guidance filter assigns high retention weight to this basin, it preserves trajectories that look locally compliant yet lie outside the recoverable region---precisely a negative retention gap, $\Gamma_{t,\lambda_t}(K_{t-1})<0$. Pseudo-compliant hallucination is therefore one manifestation of the general failure mode characterized by Lemma~\ref{lem:action_space_filtering_identity} and Theorem~\ref{thm:guidance_action_space_alignment}.

%% file: text/999_1_3_partial_harnessing.tex
\subsection{Proofs and Details for Partial Harnessing}
\label{app:partial-harnessing}

This appendix gives the formal statement and proof for the partial-harnessing principle of Section~\ref{sec:partial-harnessing}. Without additional structure, the marginal effect of adding one scaffolded stage, $F(m+1)-F(m)$, mixes the new scaffold cost, the residual tail-risk reduction, and any shift in earlier prefix distributions induced by changing the harness; the analysis isolates the first two by working on a homogeneous progress slice. Appendix~\ref{app:homogeneous-progress-slice} introduces the slice and its factorization assumption, and proves a lemma that decomposes $F(m)$ into a per-stage scaffold cost plus a residual tail cost, from which the marginal identity $F(m+1)-F(m)=c_s-\Delta(m;M)$ follows. Building on this lemma, Appendix~\ref{app:proof-thm-partial-harnessing} adds discrete convexity of the tail to prove Theorem~\ref{thm:partial-harnessing}, characterizing the unimodal success curve and the smallest scaffold count that maximizes reliability or attains a target $\alpha$. Appendix~\ref{app:connection-local-alignment} then closes the loop with the preceding two principles, expressing the scalar $c_s$ in terms of the granularity penalty and the guidance retention gap, and Appendix~\ref{app:slice-failure-modes} delineates the regimes in which the slice assumptions break and the marginal rule no longer applies.

\subsubsection{Homogeneous Progress Slice}
\label{app:homogeneous-progress-slice}

Fix a task $x$ with total latent progress demand $L_x$ and scaffold step size $s>0$, and let $N:=\lfloor L_x/s\rfloor$, $\mathcal J:=\{0,1,\ldots,N\}$. For each $m\in\mathcal J$, the partial harness $h_m$ executes $m$ scaffolded stages of length $s$ and hands the residual length $L_x-ms$ to the autonomous continuation policy. Let $A_t^{(m)}$ be the event that the $t$-th scaffolded stage under $h_m$ preserves recoverability, and let $A_{\mathrm{tail}}^{(m)}$ be the event that the autonomous tail succeeds.

\begin{assumption}[Positive homogeneous slice factorization]
\label{assump:slice-factorization}
There exist scalars $q_{\mathrm{scaf}}(s;M)\in(0,1]$ and $q_{\mathrm{tail}}(L_x-ms;M)\in(0,1]$, with $q_{\mathrm{tail}}(0;M)=1$, such that for every $m\in\mathcal J$:
\textbf{(i) goal consistency:} $\mathrm{Succ}_x(\tau_h) \equiv \bigcap_{t=1}^m A_t^{(m)} \cap A_{\mathrm{tail}}^{(m)}$;
\textbf{(ii) homogeneous scaffold reliability:} $\Pr_{h_m}(A_t^{(m)} \mid A_{<t}^{(m)},x)=q_{\mathrm{scaf}}(s;M)$ for $t=1,\ldots,m$;
\textbf{(iii) residual-tail reduction:} $\Pr_{h_m}(A_{\mathrm{tail}}^{(m)} \mid A_1^{(m)},\ldots,A_m^{(m)},x)=q_{\mathrm{tail}}(L_x-ms;M)$.
\end{assumption}

The assumption compresses prefix-level behavior into two scalar quantities: scaffolded stages have the same effective success probability along the slice, and the tail depends only on residual length, not on the realized prefix. Define
$$
c_s:=-\log q_{\mathrm{scaf}}(s;M),
\qquad
\kappa_{\mathrm{tail}}(d;M):=-\log q_{\mathrm{tail}}(d;M),
$$
both finite on the slice grid with $\kappa_{\mathrm{tail}}(0;M)=0$.

\begin{lemma}[Slice factorization and marginal identity]
\label{lem:slice-factorization}
Under Assumption~\ref{assump:slice-factorization}, for every $m\in\mathcal J$,
$$
F(m)
:=
-\log \Pr_{h_m}(\mathrm{Succ}_x(\tau_h)\mid x)
=
m c_s+\kappa_{\mathrm{tail}}(L_x-ms;M),
$$
and for $m,m+1\in\mathcal J$,
$$
F(m+1)-F(m)=c_s-\Delta(m;M),
\qquad
\Delta(m;M)
:=
\kappa_{\mathrm{tail}}(L_x-ms;M)
-
\kappa_{\mathrm{tail}}(L_x-(m+1)s;M).
$$
Adding the $(m+1)$-st scaffolded stage strictly improves reliability if and only if $\Delta(m;M)>c_s$.
\end{lemma}

\begin{proof}
By (i), $\Pr_{h_m}(\mathsf{Succ}\mid x)=\prod_{t=1}^m \Pr_{h_m}(A_t^{(m)}\mid A_{<t}^{(m)},x)\cdot\Pr_{h_m}(A_{\mathrm{tail}}^{(m)}\mid A_1^{(m)},\ldots,A_m^{(m)},x)$ via chain rule. Substituting (ii) and (iii) gives $\Pr_{h_m}(\mathsf{Succ}\mid x)=q_{\mathrm{scaf}}(s;M)^m\,q_{\mathrm{tail}}(L_x-ms;M)$, and taking negative logarithms gives $F(m)=mc_s+\kappa_{\mathrm{tail}}(L_x-ms;M)$. The marginal identity follows by direct subtraction, and the $\Delta(m;M)>c_s$ characterization from $F(m+1)<F(m)\iff c_s<\Delta(m;M)$.
\end{proof}

\subsubsection{Proof of Theorem~\ref{thm:partial-harnessing}}
\label{app:proof-thm-partial-harnessing}

\begin{assumption}[Convex tail risk on the slice]
\label{assump:convex-tail-risk}
On the grid \(\{L_x-ms:m\in\mathcal J{J}\}\), the function
\(d\mapsto \kappa_{\mathrm{tail}}(d;M)\) is finite, nondecreasing, and convex, with
\(\kappa_{\mathrm{tail}}(0;M)=0\).
\end{assumption}

\begin{proof}
By Lemma~\ref{lem:slice-factorization}, $F(m+1)-F(m)=c_s-\Delta(m;M)$.

\paragraph{Discrete convexity.}
Let $d_m:=L_x-ms$. By Assumption~\ref{assump:convex-tail-risk}, $d\mapsto\kappa_{\mathrm{tail}}(d;M)$ is convex on the slice grid, so the discrete increment $d\mapsto\kappa_{\mathrm{tail}}(d;M)-\kappa_{\mathrm{tail}}(d-s;M)$ is nondecreasing in $d$. Since $d_m$ decreases as $m$ increases, $\Delta(m;M)=\kappa_{\mathrm{tail}}(d_m;M)-\kappa_{\mathrm{tail}}(d_m-s;M)$ is nonincreasing in $m$, hence $c_s-\Delta(m;M)$ is nondecreasing. The forward difference of $F$ is nondecreasing, so $F$ is discrete-convex; equivalently $\exp(-F(m))$ is log-concave and unimodal on $\mathcal J$.

\paragraph{Smallest reliability maximizer.}
For a discrete-convex sequence, the minimizer set is an interval whose smallest element is the first index at which the forward difference becomes nonnegative. By Lemma~\ref{lem:slice-factorization}, $F(m+1)-F(m)\ge 0\iff \Delta(m;M)\le c_s$, giving
$$
m_{\mathrm{peak}}
=
\min\{m\in\mathcal J:m+1\in\mathcal J,\ \Delta(m;M)\le c_s\},
$$
with $m_{\mathrm{peak}}=\max\mathcal J$ if the set is empty.

\paragraph{Minimum scaffold for target reliability.}
For $\alpha\in(0,1)$, $P_m(x)\ge\alpha\iff F(m)\le-\log\alpha$. The feasible set $\mathcal J_\alpha=\{m\in\mathcal J:F(m)\le-\log\alpha\}$ is nonempty iff the target is achievable, in which case $m_\alpha=\min\mathcal J_\alpha$.
\end{proof}

\subsubsection{Connection to Local Alignment Results}
\label{app:connection-local-alignment}

The slice scalar $c_s=-\log\Pr(A_s)$, with $A_s$ the event that one scaffolded stage of length $s$ preserves recoverability, exposes how the local alignment results of Sections~\ref{sec:granularity_capability_alignment} and~\ref{sec:guidance_action_space_filtering} drive the global trade-off. The granularity bound certifies $\Pr(A_s)\le\exp(-\gamma_s)$ when the scaffolded sub-goal lies outside the agent's reachable scales, hence $c_s\ge\gamma_s$. The guidance identity gives $\Pr(A_s\mid K_{t-1})=\sigma(\omega^0_t(K_{t-1})+G_t(K_{t-1}))$, so a negative retention gap raises $c_s$ while a positive gap lowers it. The slice criterion $\Delta(m;M)>c_s$ therefore summarizes the two local effects as a single global stopping condition: alignment makes additional scaffolded stages cheaper, misalignment makes them more expensive.

\begin{remark}[Capability frontier in the zero-cost limit]
\label{rem:capability-frontier}
In the idealized regime $c_s=0$, the minimum-structure rule simplifies. Define $\mathcal D_\alpha:=\{d\in\{L_x-ms:m\in\mathcal J\}:\kappa_{\mathrm{tail}}(d;M)\le-\log\alpha\}$ and $d_\alpha:=\max\mathcal D_\alpha$ when nonempty. Then $m_\alpha=\min\{m\in\mathcal J:L_x-ms\le d_\alpha\}=\lceil(L_x-d_\alpha)/s\rceil_+$, recovering the capability-frontier reading: scaffold only until the residual lies within the agent's autonomous reliability range. For $c_s>0$ the exact grid rule from Theorem~\ref{thm:partial-harnessing} should be used.
\end{remark}

\subsubsection{When the Slice Assumptions Fail}
\label{app:slice-failure-modes}

The slice rule rests on two restrictions and can fail when they do. \emph{Non-nested or behavior-shifting workflows:} if $h_{m+1}$ is not a nested truncation of $h_m$, or if added later instructions alter the agent's earlier behavior (through prompt length, attention, or reinterpretation of the goal), then the first $m$ stages under $h_m$ and $h_{m+1}$ have different distributions, violating the scalar cost model $F(m)=mc_s+\kappa_{\mathrm{tail}}$. \emph{Prefix-dependent tail reliability:} if tail success depends on the realized prefix rather than only on residual length, $F(m+1)-F(m)$ contains additional terms beyond $c_s$ and $\Delta(m;M)$, and the simple rule $\Delta(m;M)>c_s$ no longer applies. \emph{Non-convex tail risk:} reliability cliffs, irreversible overshoot, discrete action scales, or prefix-dependent repair can break the convexity of $\kappa_{\mathrm{tail}}$; $\Delta(m;M)$ may then fail to be monotone in $m$ and $P_m(x)$ may be multi-modal. In any of these regimes, the safe procedure is to estimate or compare $F(m)$ directly across candidate scaffold counts.

%% file: text/999_4_prompt.tex
\section{Prompt Details}

We list the two harness templates used in our experiments below. They share a common skeleton but isolate different variables: the first holds \texttt{guidance} fixed and varies \texttt{workflow} step depth (granularity experiments), while the second holds \texttt{workflow} fixed and varies the \texttt{guidance} string (guidance experiments). Placeholders in curly braces are filled in per task at evaluation time.

\vspace{0.5em}
\begin{prompt}{Harness Template on Granularity Experiments}
\vspace{0.5em}
\textbf{Task:} \texttt{\{task\}}

\medskip
\textbf{Guidance:}
\begin{itemize}
  \item Follow the provided workflow.
  \item Inspect the evaluator and target files before choosing an implementation strategy.
  \item Reproduce the failure or run a minimal check before editing.
  \item Make the smallest local change that passes the focused evaluation without touching unrelated code.
\end{itemize}

\medskip
\textbf{Workflow:} \texttt{\{workflow\}}
\end{prompt}

\vspace{0.5em}
\begin{prompt}{Harness Template on Guidance Experiments}
\vspace{0.5em}
\textbf{Task:} Analyze the provided Plotly chart structure and produce a concise, evidence-grounded conclusion.

\medskip
\textbf{Guidance:} \texttt{\{guidance\}}

\medskip
\textbf{Workflow:}
\begin{enumerate}
  \item Identify the chart type and the available variables.
  \item State what can and cannot be inferred from the provided information.
  \item Give a concise conclusion without inventing unsupported distributional details.
\end{enumerate}
\end{prompt}